\documentclass{article} 
\usepackage{nips14submit_e,times}

\usepackage{amsmath,amsthm,amssymb}

\usepackage{color}
\definecolor{Blue}{rgb}{0.9,0.3,0.3}

\newcommand{\squishlist}{
   \begin{list}{$\bullet$}
    { \setlength{\itemsep}{0pt}      \setlength{\parsep}{3pt}
      \setlength{\topsep}{3pt}       \setlength{\partopsep}{0pt}
      \setlength{\leftmargin}{1.5em} \setlength{\labelwidth}{1em}
      \setlength{\labelsep}{0.5em} } }

\newcommand{\squishlisttwo}{
   \begin{list}{$\bullet$}
    { \setlength{\itemsep}{0pt}    \setlength{\parsep}{0pt}
      \setlength{\topsep}{0pt}     \setlength{\partopsep}{0pt}
      \setlength{\leftmargin}{2em} \setlength{\labelwidth}{1.5em}
      \setlength{\labelsep}{0.5em} } }

\newcommand{\squishend}{
    \end{list}  }

\newcommand{\real}{\mathbb{R}}

\newcommand{\tr}{\mathrm{tr}}

\newcommand{\gauss}{{\cal N}}

\newcommand{\myvec}[1]{\mathbf{#1}}
\newcommand{\myvecsym}[1]{\boldsymbol{#1}}

\newcommand{\vepsilon}{\myvecsym{\epsilon}}

\newcommand{\vlambda}{\myvecsym{\lambda}}
\newcommand{\vLambda}{\myvecsym{\Lambda}}

\newcommand{\vtheta}{\myvecsym{\theta}}

\newcommand{\vTheta}{\myvecsym{\Theta}}

\newcommand{\vSigma}{\myvecsym{\Sigma}}

\newcommand{\vf}{\myvec{f}}

\newcommand{\vk}{\myvec{k}}

\newcommand{\vm}{\myvec{m}}

\newcommand{\vx}{\myvec{x}}

\newcommand{\vy}{\myvec{y}}

\newcommand{\vI}{\myvec{I}}
\newcommand{\vJ}{\myvec{J}}
\newcommand{\vK}{\myvec{K}}

\newcommand{\vQ}{\myvec{Q}}

\newcommand{\E}{\mathbb{E}}

\newcommand{\expectAngle}[1]{\langle #1 \rangle}
\newcommand{\expect}[1]{\mathbb{E}\left[ {#1} \right]}

\newcommand{\Var}{\mathrm{Var}}

\newcommand{\diag}{\mathrm{diag}}

\newcommand{\calA}{{\cal A}}
\newcommand{\calB}{{\cal B}}

\newcommand{\calD}{{\cal D}}

\newcommand{\calH}{{\cal H}}

\newcommand{\calI}{{\cal I}}
\newcommand{\calK}{{\cal K}}

\newcommand{\calX}{{\cal X}}

\newcommand{\calO}{{\cal O}}

\newcommand{\data}{\calD}

\newcommand{\be}{\begin{equation}}
\newcommand{\ee}{\end{equation}}
\newcommand{\bea}{\begin{eqnarray}}
\newcommand{\eea}{\end{eqnarray}}
\newcommand{\beaa}{\begin{eqnarray*}}
\newcommand{\eeaa}{\end{eqnarray*}}

\DeclareMathAlphabet{\mathpzc}{OT1}{pzc}{m}{n}

\usepackage{hyperref}
\usepackage{url}

\usepackage{amsmath}
\usepackage{natbib}
\usepackage{graphicx}
\usepackage{fancyhdr}
\usepackage{color}
\usepackage{algorithm}
\usepackage{algorithmic}

\newtheorem{mydefinition1}{Definition}
\newtheorem{mydefinition2}{Definition}

\newtheorem{mydefinition6}{Definition}

\newtheorem{lemma}[mydefinition1]{Lemma}
\newtheorem{theorem}[mydefinition2]{Theorem}

\newtheorem{proposition}[mydefinition6]{Proposition}

\usepackage{array}
\newcolumntype{C}{>{\centering\arraybackslash} m{6cm} }

\DeclareMathOperator*{\argmax}{arg\,max}

\newcommand{\rnorm}[1]{\|{#1}\|_{\calH_{\vtheta}(\calX)}}
\usepackage{amsbsy}

\def\T{^\mathsf{T}}

\title{Theoretical Analysis of Bayesian Optimisation \\ with Unknown Gaussian Process Hyper-Parameters}

\author{
Ziyu Wang$^1$, 
Nando de Freitas$^{1,2}$ \\
$^1$University of Oxford\\
$^2$Canadian Institute for Advanced Research \\
\texttt{\{ziyu.wang, nando\}@cs.ox.ac.uk}
}

\nipsfinalcopy 

\begin{document}

\maketitle


\begin{abstract}
Bayesian optimisation has gained great popularity as a tool for optimising the parameters of machine learning algorithms and models. Somewhat ironically, setting up the hyper-parameters of Bayesian optimisation methods is notoriously hard. While reasonable practical solutions have been advanced, they can often fail to find the best optima. Surprisingly, there is little theoretical analysis of this crucial problem in the literature. To address this, 
we derive a cumulative regret bound for Bayesian optimisation with Gaussian processes and unknown kernel hyper-parameters in the stochastic setting. The bound, which applies to 
the expected improvement acquisition function and sub-Gaussian observation noise, provides us with guidelines on how to design hyper-parameter estimation methods. A simple simulation demonstrates the importance of following these guidelines. 
\end{abstract}

\section{Introduction}

Bayesian optimisation has become an important area of research and development in the field of machine learning, as evidenced by recent media coverage \cite{WiredSpearmint} and a blossoming range of applications to interactive user-interfaces \cite{Brochu:2010}, robotics \cite{Lizotte:2007,martinez-cantin:2007}, environmental monitoring \cite{Marchant:2012}, information extraction \cite{Wang:2014aistats}, combinatorial optimisation \cite{Hutter:smac,Wang:rembo}, automatic machine learning \cite{Bergstra:2011,Snoek:2012,Swersky:2013,Thornton:2013,Hoffman:2014}, sensor networks \cite{Garnett:2010,Srinivas:2010}, adaptive Monte Carlo \cite{Mahendran:2012}, experimental design \cite{Azimi:2012} and reinforcement learning \cite{Brochu:2009}.

In Bayesian optimisation, Gaussian processes are one of the preferred priors for quantifying the uncertainty in the objective function \cite{Brochu:2009}. However, estimating the hyper-parameters of the Gaussian process kernel with very few objective function evaluations is a daunting task, often with disastrous results as illustrated by a simple example in \cite{Benassi:2011}. The typical estimation of the hyper-parameters by maximising the marginal likelihood~\cite{Rasmussen:2006,Jones:1998} can easily fall into traps; as shown in \cite{Bull:2011}. To circumvent this, \cite{Wang:rembo} introduced adaptive bounds on the range of hyper-parameter values. 

Several authors have proposed to integrate out the hyper-parameters using quadrature and Monte Carlo methods \cite{Osborne:2009,Brochu:2010,Snoek:2012}. Despite the advantages brought in by this more sophisticated treatment of uncertainty, Bayesian optimisation can still fall in traps, as illustrated with a simple simulation example in this paper. 

To the best of our knowledge, the work of Bull \cite{Bull:2011} provides the only known regret bound for Bayesian optimisation when the hyper-parameters are unknown. By introducing lower and upper bounds on the possible range of hyper-parameters values, Bull obtains convergence rates for deterministic objective functions, when estimating the hyper-parameters by maximum likelihood. Here, we extend the work of Bull to stochastic objective functions. Our results apply to sub-Gaussian noise, \emph{e.g.}, symmetric Gaussian, Bernoulli, or uniform noise.
\section{Bayesian optimisation}
\label{sec:problem}


We consider a sequential decision approach to global optimization of smooth functions $f(\cdot): \calX \mapsto \mathbb{R}$ over an index set $\calX \subset \real^d$. At the $t$-th decision round, we select an input $\vx_t \in \calX$ and observe the value of a \emph{black-box} reward function $f(\vx_t)$. The returned value may be deterministic, $y_t = f(\vx_t)$, or stochastic, $y_t = f(\vx_t) + \epsilon_t$.
Our goal is to maximise the cumulative rewards $\sum_{t=1}^{T} f(\vx_t)$. That is, we wish to approach the performance of the optimiser $\vx^* = \argmax_{\vx \in \calX} f(\vx)$ rapidly. Since the optimiser is unknown, we have to trade-off exploitation and exploration in the search process.

This sequential optimisation approach is natural when the function does not have an obvious mathematical representation  (\emph{e.g.}, when querying people to maximize some objective) or when the function is too expensive to evaluate (\emph{e.g.}, as in control problems and automatic algorithm configuration with massive datasets and models).  

Although the function is unknown, we assume that it is smooth. It is natural to adopt a Bayesian modelling approach whereby one introduces a prior to encode our beliefs over the smoothness of the function, and an observation model to describe the data $\data_t=\{\vx_{1:t}, \vy_{1:t}\}$ up to the $t$-th round. Using these two models and the rules of probability, we derive a posterior distribution $p(f(\cdot)|\data_t)$ from which we can carry out inference about properties of $f(\cdot)$ in light of the data, such as the location of its maxima.

\subsection{Bayesian optimisation with Gaussian processes}

A Gaussian processes (GP) offer a flexible and relatively simple way of placing priors over functions; we refer the reader to \cite{Rasmussen:2006} for details on these stochastic processes. Such priors are completely characterised by a mean function $m(\cdot)$ and a covariance kernel $k(\cdot,\cdot)$ on the index sets $\calX$ and $\calX \otimes \calX$.  In particular, given any finite
collection of inputs $\vx_{1:t}$ the outputs are jointly Gaussian,
\begin{equation*}
    f(\vx_{1:t})| \vtheta \sim \gauss(\vm(\vx_{1:t}), \vK^{\vtheta}(\vx_{1:t}, \vx_{1:t})),
\end{equation*}
where $\vK^{\vtheta}(\vx_{1:t}, \vx_{1:t})_{ij}=k^{\vtheta}(\vx_i,\vx_j)$ is the covariance matrix (parametrised by $\vtheta$)
and $\vm(\vx_{1:t})_i=m(\vx_i)$ the mean vector. For convenience, we assume a zero-mean prior. 
We consider the following types of covariance
kernels
\bea
    k^{\vtheta}_\textrm{SE}(\vx,\vx')
    &=& \exp(-\tfrac12r^2) 
    \label{eqn:se}
    \\
    k^{\vtheta}_\textrm{Mat\'ern(5/2)}(\vx,\vx')
    &=& \exp(-\sqrt5 r) (1+\sqrt{5}r+\tfrac53r^2)
    \label{eqn:matern}
    \\
    \text{where }
    r &=& (\vx-\vx')\T \diag(\boldsymbol \vtheta^2)^{-1}(\vx-\vx').
    \nonumber
\eea
Both kernels are parametrised by $d$ length-scale hyper-parameters $\theta_i$. 
These kernels work well in situations where
little is known about the space in question, although the Mat\'ern tends to make
less stringent smoothness assumptions, thus making it a good fit for Bayesian
optimization.

We assume that the observations of the function at any point $\vx_t$ are corrupted by $\sigma$-sub-Gaussian noise $y_t = f(\vx_t) + \epsilon_t$. Our theoretical results cover this general type of noise, which encompasses symmetric Gaussian and Bernoulli noise. However, for ease of presentation,  we will focus on the tractable case of Gaussian noise $\epsilon_t\sim\gauss(0, \sigma^2)$ in this section. We refer the reader to \cite{Brochu:2010} for an example of discrete noise, which necessitates the introduction of approximate inference methods.

Given the data $\data_t=\{\vx_{1:t}, \vy_{1:t}\}$, the
joint distribution of the data and an arbitrary evaluation point
$\vx$ is
\begin{equation*}
    \left[\begin{matrix}
        \vy_{1:t}\\
        f(\vx)
    \end{matrix}\right] \left| \vtheta\sim
    \gauss\left(0,
    \left[\begin{matrix}
        \vK^{\vtheta}_t+\sigma^2\vI & \vk^{\vtheta}_t(\vx) \\
        \vk^{\vtheta}_t(\vx)\T & k^{\vtheta}(\vx, \vx)
    \end{matrix}\right]
    \right). \right. 
\end{equation*}
where $\vK^{\vtheta}_t=\vK^{\vtheta}(\vx_{1:t},\vx_{1:t})$ and $\vk^{\vtheta}_t(\vx)=\vk^{\vtheta}(\vx_{1:t},\vx)$. It is well known that the predictive posterior distribution of any evaluation point $\vx$
is marginally Gaussian $f(\vx)|\data_t,\vtheta \sim\gauss(\mu_t(\vx; \vtheta),
\sigma_t(\vx; \vtheta))^2$, where
\bea
    \mu_t(\vx; \vtheta) &=& \expect{f(\vx)|\data_t} = \vk^{\vtheta}_t(\vx)\T(\vK^{\vtheta}_t+\sigma^2\vI)^{-1}\vy_{1:t},
    \label{eqn:mean}
    \\
    {\cal K}_t^{\vtheta}(\vx,\vx')  &=& \expect{f(\vx)f(\vx')|\data_t} = k^{\vtheta}(\vx,\vx') - \vk^{\vtheta}_t(\vx)\T(\vK^{\vtheta}_t+\sigma^2\vI)^{-1}\vk^{\vtheta}_t(\vx')
    \label{eqn:postCov}
		\\
    \sigma_t(\vx; \vtheta)^2 &=& {\cal K}_t^{\vtheta}(\vx,\vx). 
    \label{eqn:std}
\eea


Having specified a distribution over the target function and a mechanism for updating this distribution as data arrives, we turn our attention to the problem of selecting an acquisition function $\alpha(\cdot|\data_t)$ for choosing the next query point, 
\begin{equation*}
    \vx_{t+1} = \argmax_{\vx\in\calX}\alpha(\vx|\data_t).
\end{equation*}
The choice of acquisition function is crucial. It must be efficiently computable since it will be optimized at every decision
round. More subtly, it must use the statistics of $p(f(\vx)|\data_t, \vtheta)$ to trade-off exploitation (where $\mu_t(\vx; \vtheta)$ is high) and exploration (where $\sigma_t(\vx; \vtheta)$ is high) effectively.


Although many acquisition functions have been proposed (see for example~\cite{Mockus:1982,Jones:2001,Hoffman:2011,Hennig:2012,Snoek:2012,Hoffman:2014}), the expected improvement (EI) criterion remains a default choice in popular Bayesian optimisation packages, such as SMAC and Spearmint \cite{Hutter:smac,Snoek:2012}. If
we let $\vx^+_t=\argmax_{i\leq
t}f(\vx_i; \vtheta)$ denote the current \emph{incumbent}, the EI acquisition function can be
written in closed form as
\be
    \alpha^\textrm{EI(f)}_{\vtheta}(\vx|\data_t)
    = \E[\max\{0,f(\vx) - f(\vx^+)\}|\data_t]
   = \sigma_t(\vx; \vtheta)[a\Phi(a) + \phi(a)]
    \label{eqn:eideterministic}
\ee
with $a =\frac{\mu_t(\vx; \vtheta) - f(\vx^+)}{\sigma_t(\vx; \vtheta)}$, and $\phi$ and $\Phi$ are the standard
normal density and distribution functions respectively. In the special case of $\sigma_t(\vx; \vtheta) =0$, we set $\alpha^\textrm{EI(f)}_{\vtheta}(\vx|\data_t)= 0$. The expected improvement is best understood as a family of one-step-decision heuristics \cite{Brochu:2009}, with many members in this family. While the above member is reasonable for deterministic optimization, the noise in the evaluation of the incumbent, $f(\vx^+)$, causes it to be brittle in the stochastic case. In the stochastic setting, the improvement over the best mean value $\mu_{\vtheta}^+ = \max_{\vx \in \calX} 
	\mu_t(\vx; \vtheta)$ seems to be a more reasonable alternative. For this choice, we obtain a similar expression for EI,
\be
    \alpha^\textrm{EI($\mu$)}_{\vtheta}(\vx|\data_t)
    = \E[\max\{0,f(\vx) - \mu_{\vtheta}^+\}|\data_t]
= \sigma_t(\vx; \vtheta)[u\Phi(u) + \phi(u)]
    \label{eqn:eistochastic},
\ee
where $u =\frac{\mu_t(\vx; \vtheta) - \mu_{\vtheta}^+}{\sigma_t(\vx; \vtheta)}$. In this paper, we will consider a re-scaled version of this criterion:
\be
    \alpha^\textrm{EI}_{\vtheta}(\vx|\data_t)
    = \E[\max\{0,f(\vx) - \mu_{\vtheta}^+\}|\data_t]
= \nu \sigma_t(\vx; \vtheta)[\frac{u}{\nu}\Phi(\frac{u}{\nu}) + \phi(\frac{u}{\nu})]
    \label{eqn:ei}
\ee
where $\nu$ is a parameter must be estimated. Intuitively, this parameter enables us to rescale the kernel. In the deterministic case, it plays an equivalent role to multiplying the kernel by an unknown coefficient $\nu$. (For notational simplicity, we are not making dependence of EI on $\nu$ explicitly in the expression $\alpha^\textrm{EI}_{\vtheta}(\vx|\data_t)$.)
		

\subsection{An algorithm inspired by the theory}
\begin{algorithm}[t]
\caption{Bayesian Optimization with Hyper-parameter Optimization.}
\label{alg:bohyper}
\begin{algorithmic}[1]
{
    \INPUT Threshold $t_{\sigma}>0$,  percentage of reduction parameter $p \in (0, 1)$, 
    and $c_2 > c_1 > 0$.
    \INPUT Lower and upper bounds $\vtheta^L$, $\vtheta^U$ for the hyper-parameters. 
    \INPUT Initial length scale hyper-parameter $\vtheta^L \leq \vtheta_1 \leq \vtheta^U$.
    \STATE Initialize $E = 0$
    \FOR{$t=1,2,\dots$}
        \STATE Select $\vx_t = 
            \argmax_{\vx\in\calX}\alpha^\textrm{EI}_{\vtheta_t}(\vx|\data_{t-1})$ 
        \IF{$\sigma_{t-1}^2(\vx_{t} ; \theta_t) < t_{\sigma} \sigma^2 $}
            \STATE $E = E + 1$
        \ELSE
            \STATE $E = 0$
        \ENDIF
        \STATE Augment the data $\data_t = \data_{t-1} \cup (\vx_t, y_t)$
        \IF{$E=5$}
            \STATE Restrict $\vtheta^U$ such that 
            $\theta^U_i = \max\left\{\min \left[p \max_{j}\{\theta^U_j\}, 
            \theta^U_i\right], \theta^L_i\right\}$
            \STATE $E = 0$
        \ENDIF
        \STATE Choose hyper-parameters $\vtheta_{t+1}$ such that 
            $\vtheta^L \leq \vtheta_{t+1} \leq \vtheta^U$.
        \STATE Choose hyper-parameter $\nu^{\vtheta_{t+1}}_t$ such that 
            $ c_1 \xi^{\vtheta_{t+1}}_{t+1} \leq \nu^{\vtheta_{t+1}}_{t+1}
            \leq c_2 \xi^{\vtheta_{t+1}}_{t+1}$, where $\xi^{\vtheta_{t}}_{t}$
            is defined in Equation~(\ref{eqn:xi}).
    \ENDFOR
}
\end{algorithmic}
\end{algorithm}
Our main theorem (Theorem 1) establishes sufficient conditions to guarantee that the regret of a Bayesian optimisation algorithm with EI and hyper-parameter estimation, vanishes as the number of function evaluations increases.
To illustrate the value of Theorem 1, we use its guidelines to construct an algorithm in this section.

For Theorem 1 to hold, it is necessary that we adapt the hyper-parameters in a particular manner. First, we must ensure that
that there exist upper-bounds on the hyper-parameters $\vtheta$, which we group in the vector 
$\vtheta^U$ ,such that the objective function $f(\cdot)$ is an element of the reproducing kernel Hilbert space induced by this narrower kernel  $\calH_{\vtheta^U}(\calX)$ (these spaces will be explained in Section 3.4). Figure 1 (right) shows what happens to the confidence intervals as the entries of $\vtheta^U$ shrink with $t$, by narrowing the kernel.

In practice, it is difficult to assess this condition. To surmount this difficulty, we draw inspiration from~\cite{Wang:rembo}, and propose to reduce the upper bound of
the length scales $\vtheta^U$ when the algorithm becomes over confident.
In particular, we adaptively reduce $\vtheta^U$ whenever the model repeatedly samples points
of low posterior variance in comparison to the noise variance $\sigma^2$. 
Once the algorithm optimizes to the precision of the noise variance, it suffers
from a slower convergence rate. 

By choosing to lower the upper bound as proposed in Algorithm~\ref{alg:bohyper}, 
we essentially enable the algorithm to explore more, as opposed to over-exploiting a local mode. This is illustrated in Figure 1, which depicts the result of running the proposed algorithm and a standard Bayesian optimisation scheme. We will explain the experiment in more detail at the end of this section.

As $\vtheta^U$ is successively decreased, 
after a finite number of iterations, 
we can ensure that $f(\cdot) \in \calH_{\vtheta^U}(\calX)$ 
as long as there exists $\vtheta \geq \vtheta^L$
such that $f(\cdot) \in \calH_{\vtheta}(\calX)$. In practice, we advocate a conservative choice of $\vtheta^L$ whenever we have little knowledge of the range of possible values of $\vtheta$. 

Theorem 1 also imposes a condition on $\nu$.
To satisfy it,
we constrain $\nu^{\vtheta_{t}}_t$ to be in the interval
$c_1 \xi^{\vtheta_{t}}_t \leq \nu^{\vtheta_{t}}_t 
\leq c_2 \xi^{\vtheta_{t}}_t$, where 
\bea
    \label{eqn:xi}
    \xi^{\vtheta_{t}}_t = \left(\calI_{\vtheta_{t}}(\vy_{t-1} ;\vf_{t-1}) 
    + \log^{1/2}(2t^2 \pi^2/3\delta)  
    \sqrt{\calI_{\vtheta_{t}}(\vy_{t-1} ;\vf_{t-1})} + \log(t^2 \pi^2/3\delta)\right).
\eea
(The information gain will be defined in Section 3.3.) 
The careful reader may have noticed that the above condition does not match perfectly
the condition detailed in Theorem 1. Upon closer examination, however,
we see that replacing the maximum information gain $\gamma^{\vtheta}_T$ with
$\calI_{\vtheta}(\vy_T ;\vf_T)$ does not break the convergence result.
We have used $\gamma^{\vtheta}_T$ in Theorem 1 instead of $\calI_{\vtheta}(\vy_T ;\vf_T)$ simply to simplify the presentation.

In practice, we could use a number of strategies for estimating the hyper-parameters, provided
they fall within the bounds set by Theorem 1.
In particular, we could use maximum likelihood to estimate the hyper-parameters in this constrained space.
Note that the $\nu$ parameter could also be treated as a kernel hyper-parameter (kernel scale), therefore removing the need
of estimating it separately. 

Finally, the astute reader would have noticed the parameters $t_\sigma$, $p$, $c_2$ and $c_1$ in the algorithm. If we want to achieve an accuracy comparable to the noise variance, we should set $t_\sigma=1$. The other parameters simply determine how fast the algorithm converges and should be set to reasonable fixed values, e.g. $p=0.5$, $c_2=1$ and $c_1=0.001$. Provided $t_{\sigma}>0$, $p \in (0, 1)$ 
 and $c_2 > c_1 > 0$, the theory is satisfied.

If we have strong beliefs about our GP prior model, it may seem unnecessary to
estimate our parameters with Algorithm~\ref{alg:bohyper}.
When our prior belief is misplaced, however, we could fail to converge
if we were to follow the traditional probabilistic approach. 
We provide an illustration of this effect by optimize the following 
stochastic function:
$$f(x) = 2\vk_{SE}^{\theta_1}(x_1, x) + 4\vk_{SE}^{\theta_2}(x_2, x) + \epsilon$$
over the interval $[0, 1]$,
where $\theta_1 = 0.1$, $\theta_2 = 0.01$, $x_1=0.1$, $x_2=0.9$, and $\epsilon$ 
is  zero-mean Gaussian with $10^{-2}$ standard deviation.
Figure~\ref{fig:example} compares Algorithm~\ref{alg:bohyper} against standard Bayesian optimisation with the same EI function, but using slice sampling to infer the kernel hyper-parameters (without imposing the theoretical bounds on the hyper-parameters). 
We see that, in the absence of reasonable prior beliefs, 
conditions like the ones detailed in our theoretical results 
are necessary to guarantee reasonable sampling of the objective function. (The same behaviour for the plot on the left is observed if we replace slice sampling with maximum likelihood estimation of the hyper-parameters.)
While heteroskedastic GP approaches could mitigate this problem, 
there are no theoretical results to guarantee this to the best of our knowledge.

\begin{figure}[t!]
    \begin{tabular}{lll}
        \includegraphics[width=0.42\columnwidth]{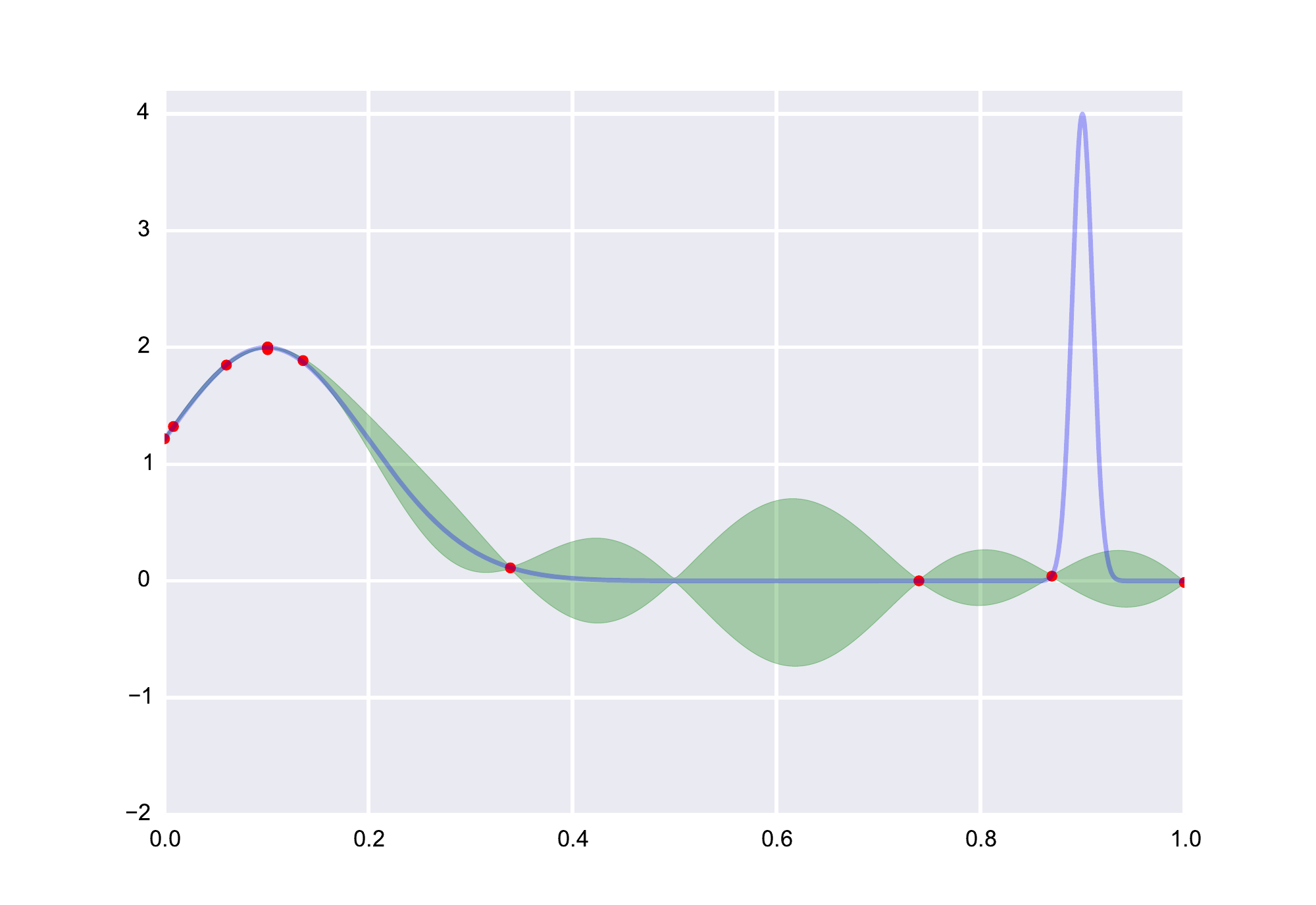}  
        & \includegraphics[width=0.42\columnwidth]{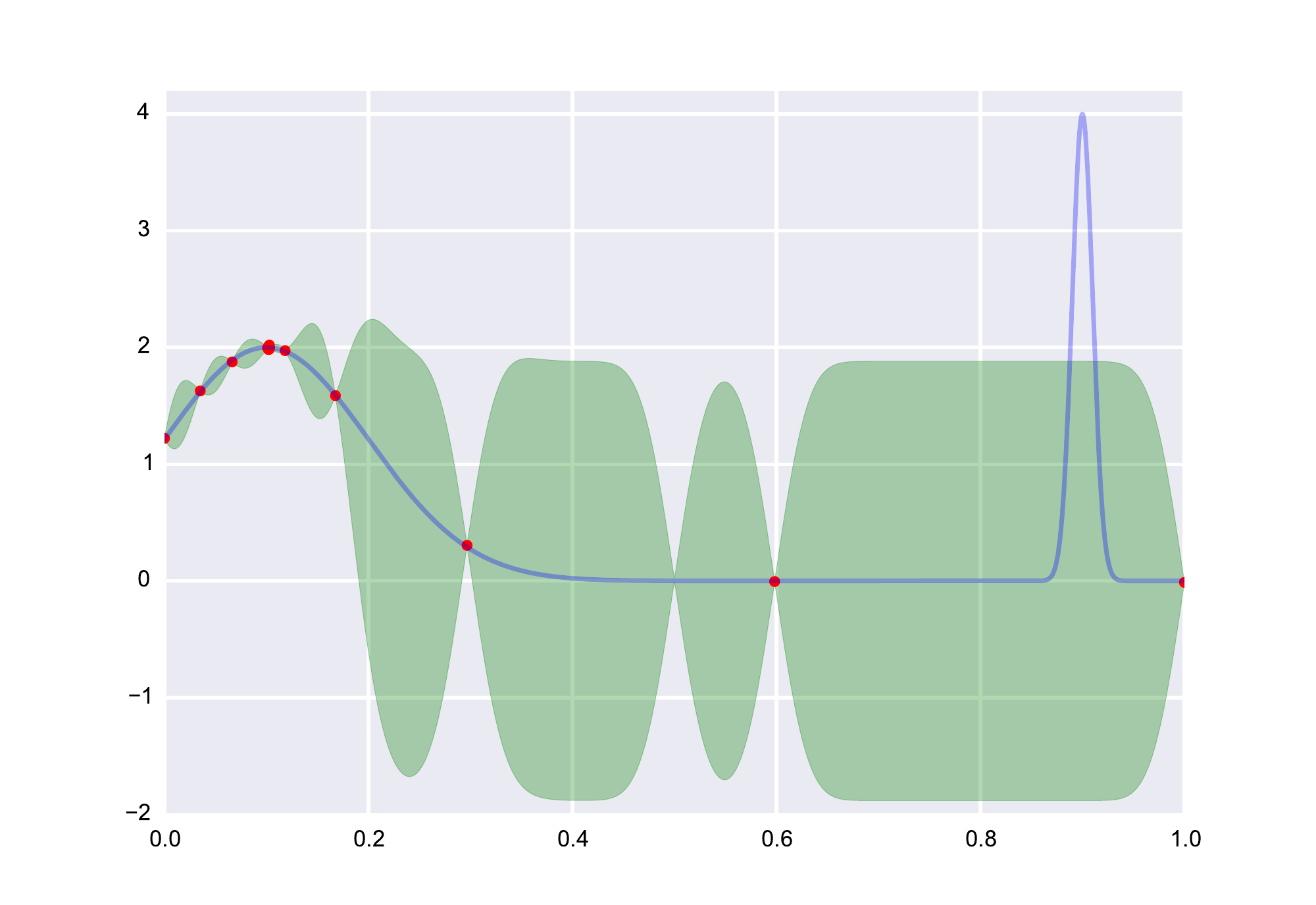} &
        \raisebox{0pt}[0pt][0pt]{\textbf{\raisebox{12ex}{$t=20$}}} \\
        \includegraphics[width=0.42\columnwidth]{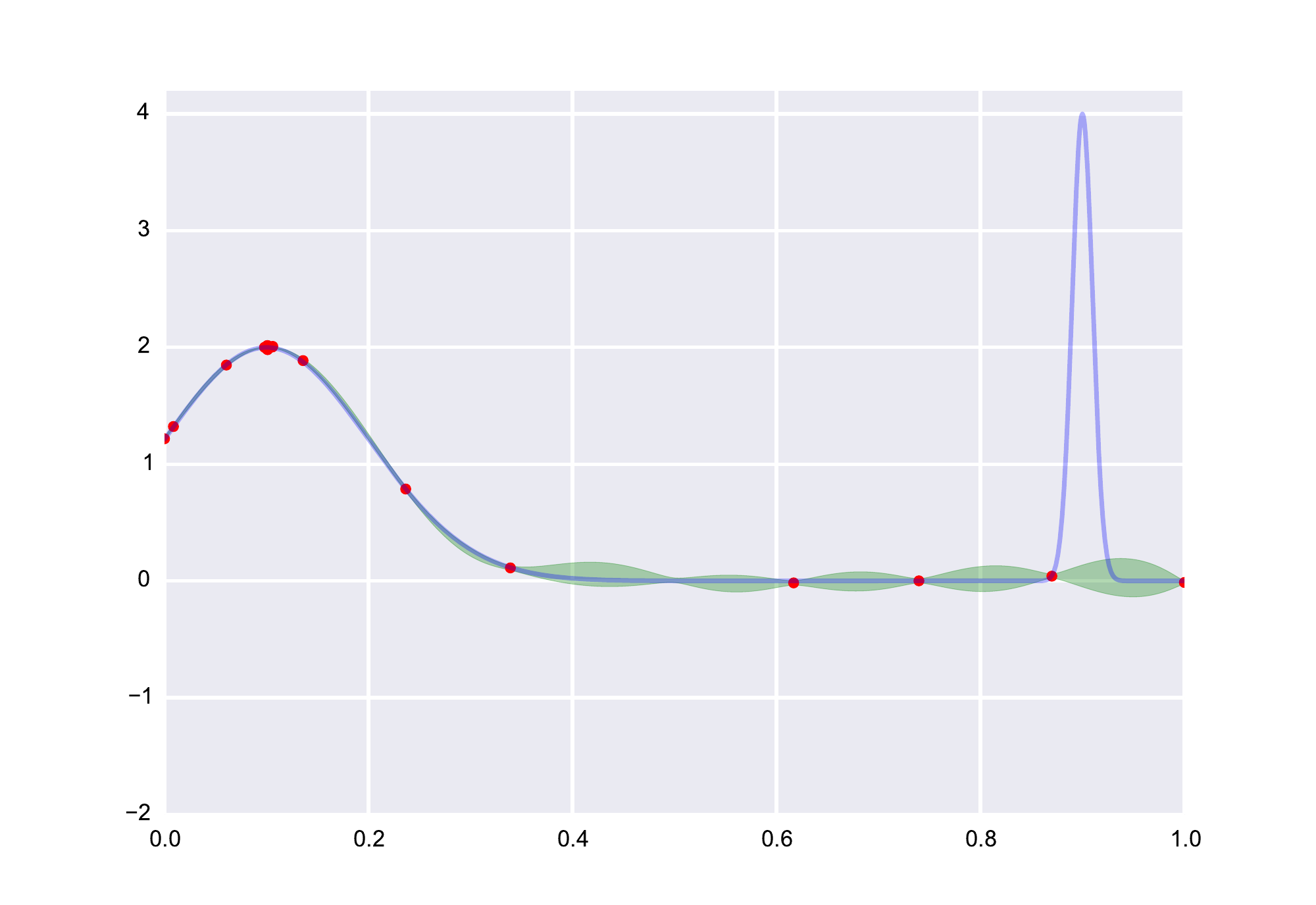}  
        & \includegraphics[width=0.42\columnwidth]{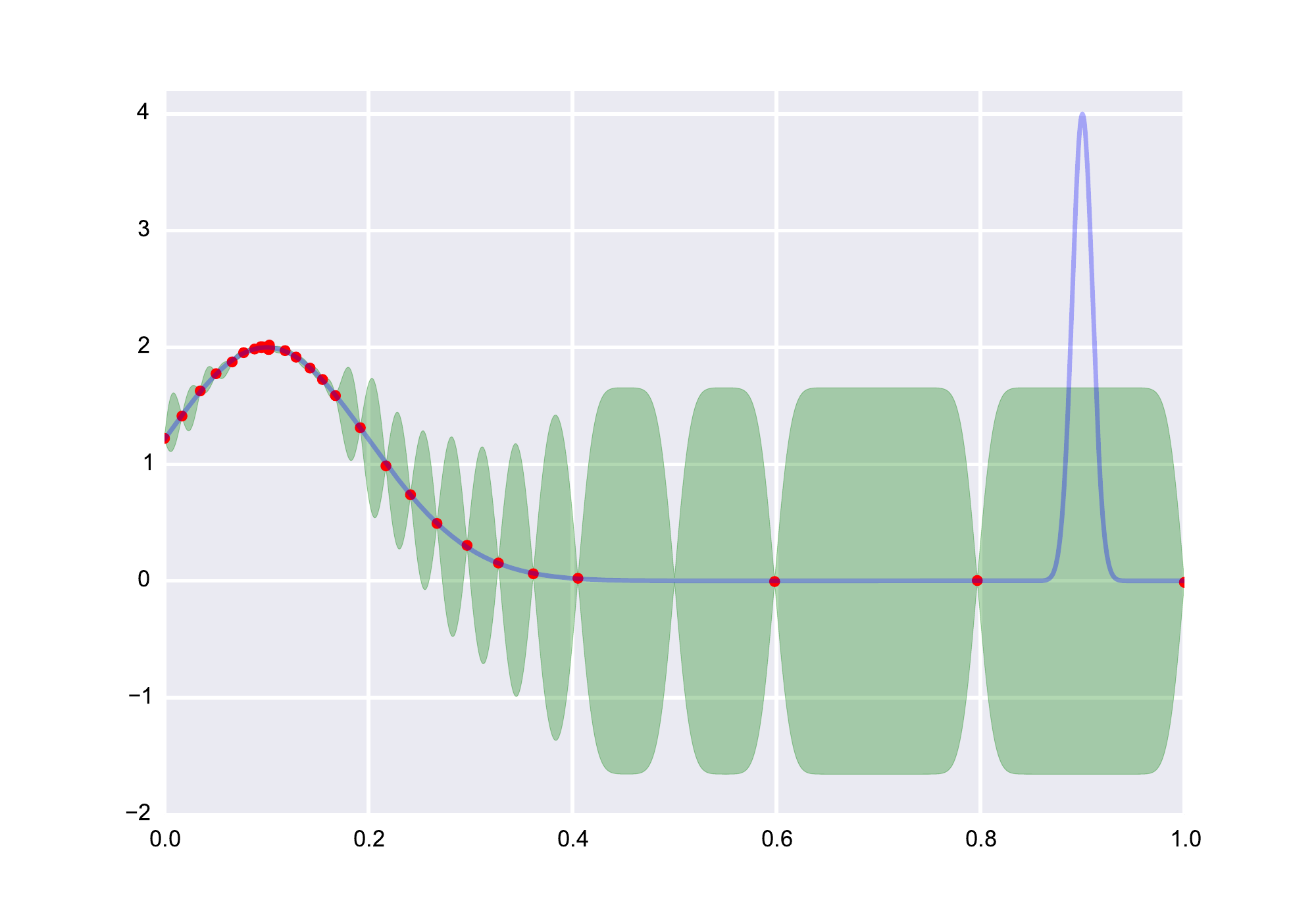} &
        \raisebox{0pt}[0pt][0pt]{\textbf{\raisebox{12ex}{$t=40$}}} \\
        \includegraphics[width=0.42\columnwidth]{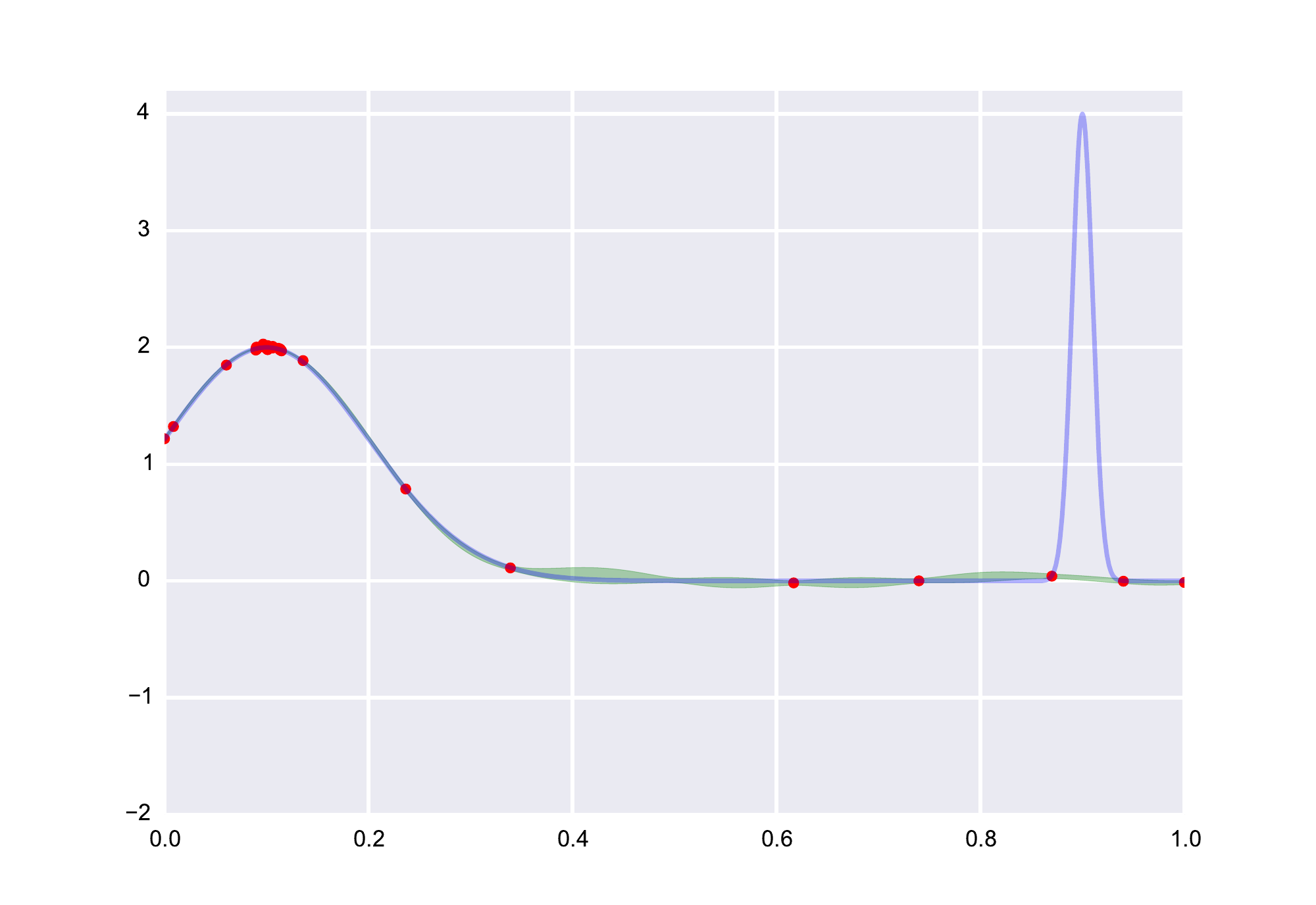}  
        & \includegraphics[width=0.42\columnwidth]{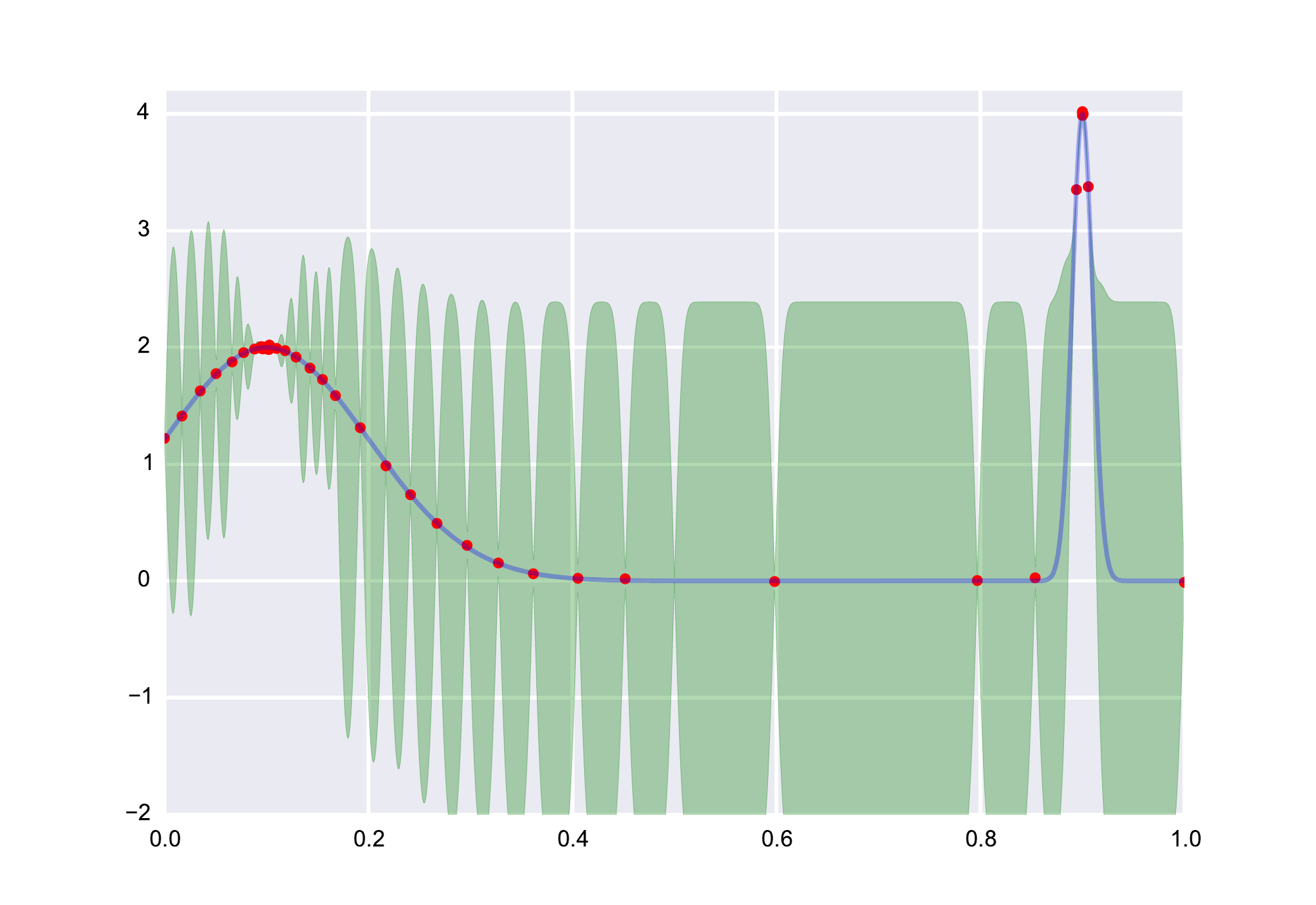} &
        \raisebox{0pt}[0pt][0pt]{\textbf{\raisebox{12ex}{$t=60$}}} \\
    \end{tabular}
    \caption{\label{fig:example}
        Convergence of EI with slice sampling over the kernel hyper-parameters [\textbf{left}] and EI using Algorithm~\ref{alg:bohyper} [\textbf{right}] at three function evaluation steps ($t$). The objective function (in blue) was constructed so that it has a trap. Unless EI with slice sampling hits the narrow optimum by random chance, it becomes too confident and fails to converge after 60 evaluations. In contrast, the confidence bounds for              
        Algorithm~\ref{alg:bohyper} can increase enabling it to sample the function in a more reasonable way and thus find the optimum.}
\end{figure}



\section{Theoretical analysis}

Our theoretical analysis uses \emph{regret} to measure convergence and \emph{information gain} to measure how informative the samples are about $f(\cdot)$. It assumes that the noise process $\epsilon_t$ is \emph{sub-Gaussian}, and that the function $f(\cdot)$ is smooth according to the \emph{reproducing kernel Hilbert space (RKHS)} associated with the GP kernel $k^{\vtheta}(\cdot,\cdot)$. 
Before presenting our main result, we briefly review these four background areas.

\subsection{Background: Regret}

As in \cite{Srinivas:2010}, we will measure the performance of the Bayesian optimization algorithm using \emph{regret}. The  
instantaneous regret at iteration $t$ is defined as $r_t = f(\vx^*) - f(\vx_t)$. The corresponding cumulative regret after $T$ iterations is $R_T = \sum_{t=1}^{T} r_t$.  
While the regret measures are never revealed to the algorithm, bounds on these enable us to assess how rapidly the algorithm is converging.

\subsection{Background: Sub-Gaussian noise}

We assume independent $\sigma$-sub-Gaussian noise. 
Formally, we say $\epsilon_t$ is $\sigma$-sub-Gaussian if there exists a $\sigma \geq 0$ such that
$$
\mathbb{E}\left[ \exp(\rho \epsilon_t) \right] \leq \exp\left( \frac{\rho^2 \sigma^2}{2} \right)
\mbox{ }
\forall \rho \in \mathbb{R}.
$$
In other works, $\epsilon_t$ is $\sigma$-sub-Gaussian if its Laplace transform is dominated by the Laplace transform of a Gaussian random variable with zero mean and variance $\sigma^2$. 
It is easy to show that if $\epsilon_t$ is sub-Gaussian, then
$\mathbb{E}[\epsilon_t] = 0$ and $\Var[\epsilon_t] \leq \sigma^2$.

There are many examples of sub-Gaussian variables, including zero-mean Gaussian random variables 
with variance $\sigma^2$, symmetric Bernoulli random variables and symmetric uniform distributions.

\subsection{Background: Information gain}

To measure the reduction in uncertainty about $f(\cdot)$ from observing $\vy_{\calA}$ for a set of sampling points $\calA \subset \calX$, we need to introduce the concept of \emph{information gain}, which is defined as the mutual information between $f(\cdot)$ and a set of observations $\vy_{\calA}$:
\be
\calI(\vy_{\calA};f(\cdot)) = H(\vy_{\calA}) - H(\vy_{\calA}|f(\cdot)).
\ee 
This concept plays a central role in the results of \cite{Srinivas:2010}, who also define the \emph{maximum information gain} $\gamma_T$ after $T$ decision rounds as
\be
\gamma_T = \max_{\calA \subset \calX: |\calA|=T} \calI(\vy_{1:t};f(\cdot)).
\ee
Note that for Gaussian distributions, 
\be
\gamma_T^{\vtheta} = \max_{\calA \subset \calX: |\calA|=T} \frac{1}{2} \log|\vI + \sigma^{-2}\vK^{\vtheta}_{\calA}|.
\ee
Our regret bounds will be given in terms of $\gamma_T^{\vtheta}$. It should perhaps be clarified that the bounds apply to $\sigma$-sub-Gaussian noise, despite the appearance of the variable $\gamma_T^{\vtheta}$ in their statements.

\subsection{Background: Reproducing kernel Hilbert spaces}

To discuss convergence, we must state formally what we mean by $f(\cdot)$ being smooth. In short, we assume that $f(\cdot)$ is an element of an RKHS with reproducing kernel $k(\cdot,\cdot)$. For an intuitive grasp of this formalisation of smoothness, we need to briefly review some RKHS fundamentals. These fundamentals are also evoked in our proofs.

Let $L_{\vx}$ be an evaluation functional: $L_{\vx} f(\cdot) = f(\vx)$. 
A (real) RKHS $\calH$ is a Hilbert space of real valued functions with the property that for each $\vx \in \calX$, the evaluation functional is bounded. That is, there exists a positive constant $M=M_{\vx}$ such that 
$
|L_{\vx} f(\cdot)| = |f(\vx)| \leq M \|f(\cdot)\|_{\calH} 
$
for all functions $f(\cdot) \in \calH$, where $\|\cdot\|_{\calH}$ denotes the norm in the Hilbert space. If $\calH$ is an RKHS, by the Riesz Representation Theorem, there exists an element $k(\cdot,\vx) \in \calH$ with the property,
\be
f(\vx) = L_{\vx} f(\cdot) = \expectAngle{k(\cdot,\vx),f(\cdot)}
\ee
for all $\vx \in \calX$ and $f(\cdot) \in \calH$, where $\expectAngle{\cdot,\cdot}$ denotes the inner product in $\calH$. 

To construct $\calH$, we consider the linear manifold $\sum_{t=1}^{n} \lambda_{t} k(\cdot,\vx_t)$ for all choices of $n$, $\lambda_1, \ldots, \lambda_n$ and $\vx_1, \ldots, \vx_n \in {\calX}$, with inner product
\be
\expectAngle{\sum_{i=1}^{n} \lambda_{i} k(\cdot,\vx_i),\sum_{j=1}^{n} \lambda_{j} k(\cdot,\vx_j)}
= \sum_{i=1}^{n} \sum_{j=1}^{n}\lambda_{i} k(\vx_i,\vx_j) \lambda_{j} 
= \| \sum_{i=1}^{n} \lambda_{i}  k(\cdot,\vx_i) \|^2_{\calH} \geq 0.
\ee
The above norm is non-negative because of the positive-definiteness of $k(\cdot,\cdot)$. Clearly, for any element $f(\cdot)$ of this linear manifold,
\be
f(\vx_j)=
\expectAngle{f(\cdot),k(\cdot,\vx_j)} = 
\expectAngle{\sum_{i=1}^{n} \lambda_{i} k(\cdot,\vx_i),k(\cdot,\vx_j)}
=\sum_{i=1}^{n} \lambda_{i} k(\vx_i,\vx_j)
\ee 
A consequence of this is that for any Cauchy sequence $\{f_n(\cdot)\}$, we have the following bound by Cauchy-Schwartz:
$
|f_n(\vx) - f(\vx)| = \expectAngle{f_n(\cdot) - f(\cdot),k(\cdot,\vx)} \leq  \|f_n(\cdot) - f(\cdot)\|_{\calH}  \|k(\cdot,\vx)\|_{\calH}.
$
In words, norm convergence implies point-wise convergence. 

The preceding steps illustrate that we can construct a unique RKHS for any positive definite kernel $k(\cdot,\cdot)$. The converse is also true (Moore-Aronszajn Theorem).

A positive definite function $k(\cdot,\cdot)$, under general conditions, has an eigenvector-eigenvalue decomposition. Suppose $k(\cdot,\cdot)$ is continuous and
$
\int_{\calX} \int_{\calX} k^{2}(\vx,\vy)d\vx d\vy < \infty,
$
then there exists an orthonormal sequence of continuous eigenfunctions $q_1(\cdot), q_2(\cdot), \ldots$ and eigenvalues $\Delta_1 \geq \Delta_2 \geq \ldots \geq 0$, with
$
k(x,y) = \sum_{\nu=1}^{\infty} \Delta_{\nu} q_{\nu}(\vx) q_{\nu}(\vy).
$
Next, consider the orthonormal expansion
$
f(\cdot) = \sum_{\nu=1}^{\infty} f_{\nu} q_{\nu}(\cdot)
$
with coefficients
$
f_{\nu} = \int_{\calX} f(\vx) q_{\nu}(\vx) d\vx.
$
It is easy to prove that $f(\cdot)$ is an element of the RKHS associated with $k(\cdot,\cdot)$ if and only if
\be
\|f(\cdot)\|^2_{\calH} = \sum_{\nu=1}^{\infty} \frac{f_{\nu}^2}{\Delta_{\nu}} < \infty.
\ee

To obtain the above finiteness condition, the coefficients $f_{\nu}$ of the expansion of $f(\cdot)$ must decay quickly. For the kernels we consider in this paper, elements of RKHS can uniformly
approximate any continuous function with compact support. 
Therefore, RKHS is well suited as a tool for analyzing
convergence behaviors of Bayesian optimization algorithms.

\subsection{Main result}

In this section, we present our regret bound and sketch its proof. For space considerations, 
detailed proofs appear in the appendix provided in the supplementary material. 

As discussed when presenting the algorithm, our theorem assumes bounds on the kernel hyper-parameters of the form $\vtheta^L \leq \vtheta_t \leq \vtheta^U$  for all $t \geq 1$ with $f(\cdot) \in \calH_{\vtheta^U}(\calX)$. While we could recall all the conditions on the kernel function necessary for our theorem to apply, we simply restrict the family of kernels to one that satisfies the conditions detailed in~\cite{Bull:2011}. Without loss of generality, we assume that $k(\vx,\vx)$= 1.

Our theorem characterising the growth in the cumulative regret $R_T$ with the number of function evaluations $T$ follows.

\begin{theorem} 
	\label{thm:main}
	Let $C_2 := \prod_{i=1}^{d} \frac{\theta^U_i}{\theta^L_i}$.
	Suppose $\vtheta^L \leq \vtheta_t \leq \vtheta^U$  for all $t \geq 1$
	and $f(\cdot) \in \calH_{\vtheta^U}(\calX)$.
	If $\left(\nu_{t}^{\vtheta}\right)^2 = \vTheta \left( \gamma^{\vtheta}_{t-1} +  
		 \log^{1/2}(2t^2 \pi^2/3\delta) 
		\sqrt{\gamma^{\vtheta}_{t-1}} + \log(t^2 \pi^2/3\delta) \right)$ for all $t \geq 1$.
	Then with probability at least $1-\delta$, the cumulative regret obeys the following rate:
	\be
	R_T = \calO\left( \beta_{T} \sqrt{\gamma^{\vtheta^L}_T T}\right), 
	\ee
	where $\beta_{T} = 2\log\left(\frac{T}{\sigma^2}\right)\gamma^{\vtheta^L}_{T-1} +  
	\sqrt{8} \log\left(\frac{T}{\sigma^2}\right) \log^{1/2}(4 T^2 \pi^2/6\delta) 
	\left(\sqrt{C_2} \|f\|_{\calH_{\vtheta^U}(\calX)} + \sqrt{\gamma^{\vtheta^L}_{T-1}}\right) + C_2\|f\|^2_{\calH_{\vtheta^U}(\calX)}.$ 
\end{theorem}
Our result is analogous to Theorem 3 of~\cite{Srinivas:2010} 
which proves convergence rates for the
GP-UCB algorithm in the agnostic setting.
Their result, however, does not allow for the estimation of hyper-parameters.
In addition, our algorithm does not require explicit knowledge of 
the RKHS norm of the objective function while GP-UCB does require this.

Using the results of Srinivas \emph{et al.}\, we can further detail these rates as follows.
\begin{theorem}[Theorem 5 of~\cite{Srinivas:2010}]
	Let $\calX \subseteq \mathbb{R}^d$ be compact and convex, $d \in \mathbb{N}$.
	Assume the kernel function satisfies $k(\vx, \vx') \leq 1$.
	\begin{enumerate}
		\item Exponential spectral decay. For the squared Exponential kernel:
		$\gamma^{\vtheta}_T = \mathcal{O}\left((\log T)^{d+1}\right).$
		\item Power law spectral decay. For Mat\'ern kernels with degree of freedom 
		$\nu > 1$:
		$\gamma^{\vtheta}_T = \mathcal{O}\left(T^{d(d+1)/(2\nu + d(d+1))} \log T\right).$
	\end{enumerate}
\end{theorem}

The proof of Theorem~\ref{thm:main} is provided in the appendix. We sketch the main ideas here. 
Our proof methodology is inspired by the works of~\cite{Srinivas:2010} and~\cite{Bull:2011}.

We start the proof-sketch by considering the instantaneous regret:
	\begin{eqnarray}
		r_t &=& f(\vx^*) - f(\vx_t) \nonumber \\
		&=& ( f(\vx^*) - \mu_{\vtheta_{t}}^+) - 
			(f(\vx_t) - \mu_{\vtheta_{t}}^+) \nonumber  
	\end{eqnarray}
where
$\mu^{+}_{\vtheta_t} = \max_{\vx \in \calX} \mu_{t-1} (\vx; \vtheta_t)$ and $\vx_t = \argmax_{\vx \in \calX} 
\alpha^\textrm{EI}_{\vtheta_{t}}(\vx|\data_{t-1})$. T

The first challenge of the proof is to bound the difference between 
the posterior mean of the GP and the objective function. 
Such a bound allows us to quantify the difference between our belief about the
objective function and the true objective.
Specifically, we bound 
$|\mu_{t-1}(\vx, \vtheta_{t}) - f(\vx)|$ $\forall \vx \in \calX$
in each iteration with high probability.
By way of the Cauchy-Schwarz inequality, 
\bea
	|\mu_{t-1}(\vx, \vtheta_{t}) - f(\vx)| 
	&\leq& \left(\calK^{\vtheta_{t}}_{t-1}(\vx, \vx) \right)^{1/2}\|\mu_{t-1}(\cdot; \vtheta) - 
		f(\cdot)\|_{\calK_{t-1}^{\vtheta_t}} \nonumber \\
	&\leq& \sigma_{t-1}(\vx; \vtheta_{t})\|\mu_{t-1}(\cdot; \vtheta) - 
		f(\cdot)\|_{\calK_{t-1}^{\vtheta_t}}.\nonumber 
\eea
The first part of our proof (Section~\ref{sec:conc} in the appendix) 
is then dedicated to providing a probabilistic bound for
$\|\mu_{t-1}(\cdot; \vtheta) - f(\cdot)\|_{\calK_{t-1}^{\vtheta_t}}$
by means of concentration inequalities.
In more detail, Lemma~\ref{lem:c1} and~\ref{lem:c2} bound separate terms
that appear in $\|\mu_{t-1}(\cdot; \vtheta) - f(\cdot)\|_{\calK_{t-1}^{\vtheta_t}}$
using properties of reproducing kernel Hilbert spaces and concentration results for sub-Gaussian random variables \cite{Hsu:2012}.
Proposition~\ref{prop:bound} combines the aforementioned results via a union bound.

The second challenge of the proof is to relate EI with quantities that are easier 
to analyse, such as the posterior variance and the improvement function  $I^{\vtheta}_t(\vx) = \max\{0,f(\vx) - \mu_{\vtheta}^+\}$. To bound the instantaneous regret, we observe that 
$$(f(\vx^*) - \mu_{\vtheta_{t}}^+) - 
(f(\vx_t) - \mu_{\vtheta_{t}}^+)
\leq I^{\vtheta}_t(\vx)  +
\left[ (\mu_{\vtheta_{t}}^+ - \mu_{t-1}(\vx_t; \vtheta_{t})) + 
\varphi_{t}^{\vtheta_{t}} \sigma_{t-1}(\vx_t; \vtheta_{t}) \right].$$
(Here $\varphi_{t}^{\vtheta_{t}}$ is a quantity that arises 
in the concentration bound of Proposition~\ref{prop:bound}.) The improvement function is upper-bounded by 
a constant times the expected improvement $\alpha^\textrm{EI}_{\vtheta_{t}}(\vx_t|\data_{t-1})$ via Lemma~\ref{lem:eikey}, which builds on results by~\cite{Bull:2011}. The expected improvement is in turn bounded by a multiple of the posterior standard deviation $\sigma_{t-1}(\vx_t; \vtheta_{t})$.

 Next, we turn our attention to the term 
$(\mu_{\vtheta_{t}}^+ - \mu_{t-1}(\vx_t; \vtheta_{t}))$. We bound this term in Lemma~\ref{lem:10}, which states that
$\left|\mu_{t-1}(\vx_{t}; \vtheta_t) - \mu_{\vtheta_{t}}^+\right| \leq
	\sqrt{\log(t-1+\sigma^2) - \log(\sigma^2)} \nu \sigma_{t-1}(\vx_t; \vtheta_t)$.

By now, we have bounded the instantaneous regret in each iteration
by a multiple of the posterior variance $\sigma_{t-1}^2(\vx; \vtheta^L)$.  Finally, we can sum over $t$ and use
Lemma~\ref{lem:54}, which states that
	$\sum_{t=t_0}^{T} \sigma_{t-1}^2(\vx; \vtheta^L) 
	\leq \frac{2}{\log(1+\sigma^{2})} \gamma^{\vtheta^L}_T$, and subsequently bound the cumulative regret by the maximal information gain, which as we said is related to the posterior variance of the GP (Lemma 5).
To accommodate different hyper-parameters, we make use of Lemma~\ref{lem:small} from~\cite{Bull:2011}.

\section{Conclusion}

Despite the rapidly growing literature on Bayesian optimisation and the proliferation of software packages that learn the kernel hyper-parameters, to the best of our knowledge, only Bull \cite{Bull:2011} and us have attacked the question of 
convergence of GP-based Bayesian optimisation with unknown hyper-parameters. Bull's results focused on deterministic objective functions. Our new results apply to the abundant class of noisy objective functions. 

\newpage
{\small
\bibliography{bayesopt}
\bibliographystyle{plain}
}

\newpage
\appendix
\newpage
\section{Proofs}

\subsection{Concentration}
\label{sec:conc}
\begin{lemma}
	\label{lem:subb}
	If $\epsilon_T$ is $\sigma$-sub-Gaussian, 
	then $\mathbb{P}(|\epsilon_T| \geq a) \leq 2\exp\left(-\frac{a^2}{2\sigma^2} \right)$
	$\forall a > 0$.
\end{lemma}
\begin{proof}
By Markov's inequality, we can see that $\forall \rho > 0$
\begin{eqnarray}
\mathbb{P} \left( \epsilon_T \geq a\right) 
= \mathbb{P}\left[\exp(\rho\epsilon_T) \geq \exp(\rho a) \right] 
\leq \frac{\mathbb{E}\left[\exp(\rho\epsilon_T)\right]}{\exp(\rho a)} 
\leq \exp\left(\frac{\sigma^2\rho^2}{2} - \rho a\right). \nonumber
\end{eqnarray}
By taking $\rho = \frac{a}{\sigma^2}$, we have that 
$\mathbb{P}(\epsilon_T \geq a) \leq \exp\left(-\frac{a^2}{2\sigma^2} \right)$.
By symmetry, we have that 
$\mathbb{P}(|\epsilon_T| \geq a) \leq 2\exp\left(-\frac{a^2}{2\sigma^2} \right)$.
\end{proof}

\vspace{4mm} 
\begin{lemma}
	\label{lem:sub2}
	Let $\epsilon_t$ be independently $\sigma$-sub-Gaussian with $t \in \{1, \cdots, T\}$. 
	Then, $\sum_{t=1}^{T} \lambda_t \epsilon_t$ is $(\|\vlambda\|\sigma)$-sub-Gaussian.
\end{lemma}
\begin{proof}
	For all $\rho \in \mathbb{R}$, we have 
	\begin{eqnarray}
	\mathbb{E}\left[\exp\left(\rho \sum_{t=1}^{T} \lambda_t \epsilon_t\right) \right] 
	&=& \mathbb{E}\left[\prod_{t=1}^{T}\exp\left(\rho \lambda_t \epsilon_t\right) \right] 
	= \prod_{t=1}^{T}\mathbb{E}\left[\exp\left(\rho \lambda_t  \epsilon_t\right) \right] \nonumber \\
	&\leq& \prod_{t=1}^{T} \exp\left( \frac{\rho^2 \lambda_t^2 \sigma^2}{2} \right)
	= \exp\left( \frac{ \rho^2 \sigma^2 \sum_{t=1}^{T} \lambda_t^2 }{2} \right). \nonumber
	\end{eqnarray}
\end{proof}

%

To shorten the notation, in the remainder of this paper, 
we will use $\vf_{T}$ to denote the vector $(f_1,\ldots,f_T) = (f(\vx_1), \ldots, f(\vx_T))$ and, similarly, we use $\vepsilon_T$ in place of $\vepsilon_{1:T}$ and $\vy_T$ in place of $\vy_{1:T}$.

\vspace{4mm}
\begin{lemma}
	\label{lem:c1}
	$\vf^T_{T}(\vK^{\vtheta}_T + \sigma^2 \vI)^{-1}\vepsilon_T$ is 
	 $\left( \rnorm{f} \right)$-sub-Gaussian.
\end{lemma}
\begin{proof}
	Consider the optimization problem
	\begin{eqnarray}
	\label{eqn:opt1}
	\min_{g \in \calH_{\vtheta}(\calX)} \sum_{t=1}^{T} \left[g(\vx_t) - f_t \right]^2 
	+ \sigma^2 \|g\|^2_{\calH_{\vtheta}(\calX)}.
	\end{eqnarray}
	By the Representer Theorem of the RKHS, we know that $g(\vx)=\vlambda^T \vk_T^{\vtheta}(\vx)$. (We remind the reader of our notation: $\vK^{\vtheta}_T=\vK^{\vtheta}(\vx_{1:T},\vx_{1:T})$ and $\vk^{\vtheta}_T(\vx)=\vk^{\vtheta}(\vx_{1:T},\vx)$.) The preceding optimisation problem is therefore equivalent to the following one:
	\begin{eqnarray}
	\label{eqn:opt}
	\min_{\vlambda} \sum_{t=1}^{T} \left[\vlambda^T \vk_T^{\vtheta}(\vx_t) - f_t \right]^2 
	+ \sigma^2 \vlambda^T K^{\vtheta}_T \vlambda.
	\end{eqnarray}
	The optimizer of~(\ref{eqn:opt}) is 
	$\vlambda = \vf_T^T (K^{\vtheta}_T + \sigma^{2} \vI)^{-1}$, with optimum value
	$\sigma^2 \vf_T^T (K^{\vtheta}_T + \sigma^{2} \vI)^{-1}\vf_T$.
Using Lemma~\ref{lem:sub2}, with $\vlambda = \vf^T_T(\vK^{\vtheta}_T + \sigma^2 \vI)^{-1}$, we notice that we only need to  
 bound $\vlambda^T\vlambda$. Proceeding,  
\bea
\vlambda^T\vlambda &=& \tr(\vf^T_{T}(K^{\vtheta}_T + \sigma^{2} \vI)^{-2}\vf_T) \\
&\leq & \frac{1}{\sigma^2} \vf^T_{T}(K^{\vtheta}_T)^{-1}\vf_T \\
& \leq & \frac{1}{\sigma^2} \rnorm{f}^2
\eea
The first inequality follows by choosing a constant $C_1 = \frac{1}{\sigma}$ so that the quadratic term $(K^{\vtheta}_T)^2_{i,j}$ upper-bounds the linear term $C_1^2(K^{\vtheta}_T)_{i,j}$.
The last inequality holds because of the fact
	that $\vf_T^T (K^{\vtheta}_T)^{-1}\vf_T$ is the minimum value 
	for the optimization problem:
	\begin{eqnarray*}
	\min_{g \in \calH_{\vtheta}(\calX)} && \rnorm{g}^2 \\
	\mbox{s.t. } && g(\vx_t) = f_t \mbox{ for } t = 1, \dots, T
	\end{eqnarray*}
	for which $f$ satisfies the constraint. That is, the function $g(\cdot)$ that agrees with $f(\vx_t)$ has minimum norm $\vf_T^T (K^{\vtheta}_T)^{-1}\vf_T$. Hence, any other function $f(\cdot)$ that agrees with $f(\vx_t)$ must have equal or larger norm.
\end{proof}

\vspace{4mm}
\begin{lemma}
	\label{lem:c2}
	$\mathbb{P}\left( \sigma^{-2}\|\vepsilon_T\|^2 - 
	\vepsilon^T_T(\vK^{\vtheta}_T + \sigma^2 \vI)^{-1}\vepsilon_T
	> 2\gamma^{\vtheta}_T + 2\sqrt{2\gamma^{\vtheta}_T \eta} + 2\sigma \eta \right)
	\leq e^{-\eta}$ for any $\eta >0$.
\end{lemma}
\begin{proof}
First, by rearrangement we have that: 
\begin{eqnarray*}
\sigma^{-2}\|\vepsilon_T\|^2 - 
\vepsilon^T_T(\vK^{\vtheta}_T + \sigma^2 \vI)^{-1}\vepsilon_T
=
\vepsilon^T_T \left[\sigma^{-2} \vI - (\vK^{\vtheta}_T + \sigma^2 \vI)^{-1} \right]\vepsilon_T 
=& \vepsilon^T_T  \vQ^T \vSigma \vQ \vepsilon_T 
\end{eqnarray*}
In the above equation $\vQ^T\vSigma \vQ$ is the eigenvalue decomposition of the 
matrix $\vLambda := \left[\sigma^{-2} \vI - (\vK^{\vtheta}_T + \sigma^2 \vI)^{-1} \right]$
where $\vQ$ is an orthonormal matrix and $\vSigma$ is a diagonal matrix.

The diagonal entries of 
$\Lambda$ are such that
$\Sigma_{i, i} = \frac{\Delta_i}{\sigma^2(\Delta_i + \sigma^2)}$ 
where $\Delta_i$ is the $i^{th}$ eigenvalue of $\vK^{\vtheta}_T$.
We know that $\tr(\vLambda) = \tr(\vSigma) 
= \sum_{i=1}^{T}\frac{\Delta_i}{\sigma^2(\Delta_i + \sigma^2)}$
and $\tr(\vLambda^2) = \tr(\vSigma^2) 
= \sum_{i=1}^{T}\left(\frac{\Delta_i}{\sigma^2(\Delta_i + \sigma^2)}\right)^2.$
It is easy to see that
$\tr(\vLambda) 
\leq \sigma^{-2} \sum_{i=1}^{T}\log(1 + \sigma^{-2}\Delta_i)$
since 
$\frac{\Delta_i}{\sigma^2(\Delta_i + \sigma^2)} \leq 
\sigma^{-2}\log(1 + \sigma^{-2}\Delta_i)$
for all $1 \leq i \leq T$.
Also 
$\tr(\vLambda^2)
\leq \sigma^{-4}\sum_{i=1}^{T}\log(1 + \sigma^{-2}\Delta_i)$
since $\frac{\Delta_i}{\sigma^2(\Delta_i + \sigma^2)} < \sigma^{-2}$
for all $1 \leq i \leq T$.
Finally, $\|\vLambda\|_2 = 
\max_{1 \leq i \leq T} \sqrt{\frac{\Delta_i}{\sigma^2(\Delta_i + \sigma^2)}}
\leq \sigma^{-1}$ again because of the fact that
$\frac{\Delta_i}{\sigma^2(\Delta_i + \sigma^2)} \leq \sigma^{-2}$.

Using the definition of maximum information gain $\gamma_{T}^{\theta}$ for Gaussians, we have the following
three facts:
\bea
\tr(\vLambda) &\leq& 2\sigma^{-2}\gamma^{\vtheta}_T \\
\tr(\vLambda^2) &\leq& 2\sigma^{-4}\gamma^{\vtheta}_T \\
\|\vLambda\|_2 &\leq& \sigma^{-1}.
\eea
By Theorem 2.1 of~\cite{Hsu:2012}, we have that
\bea
\mathbb{P}\left( \vepsilon^T_T\vLambda\vepsilon_T 
> 2\gamma^{\vtheta}_T + \
2\sqrt {2\gamma^{\vtheta}_T \eta} + 2\sigma \eta \right)
&\hspace{-2mm}=\hspace{-2mm}&
\mathbb{P}\left( \vepsilon^T_T\vLambda\vepsilon_T 
> \sigma^{2} (2 \sigma^{-2}\gamma^{\vtheta}_T + \
2\sqrt{2\sigma^{-4}\gamma^{\vtheta}_T \eta} + 2\sigma^{-1} \eta) \right) \nonumber \\
& \hspace{-2mm}\leq\hspace{-2mm} & 
\mathbb{P}\left( \vepsilon^T_T\vLambda\vepsilon_T 
> \tr(\vLambda) + \
2\sqrt{\tr(\vLambda^2) \eta} + 2 \|\vLambda\| \eta \right)
\\
& \hspace{-2mm}\leq\hspace{-2mm} &  e^{-\eta}
\end{eqnarray}
which concludes the proof.
\end{proof}

\vspace{4mm}
\begin{proposition}
	\label{prop:bound}
	Let $(\varphi_{T}^{\vtheta})^2 = \|f\|^2_{\calH_{\vtheta}(\calX)} +  
	\sqrt{8\gamma^{\vtheta}_{T-1} \log(\frac{T^2\pi^2}{3\delta})  } +
	\sqrt{2\log(\frac{2 T^2\pi^2}{3\delta})}  \rnorm{f}
		+ 2\gamma^{\vtheta}_{T-1}  + 2\sigma \log(\frac{ T^2\pi^2}{3\delta})$. 
	Then 
	$
	\mathbb{P}\left(\|\mu_T(\cdot; \vtheta) - 
	f(\cdot)\|_{\calK_T^{\vtheta}} \leq \varphi_{T+1}^{\vtheta} \right) 
	\geq 1 - \frac{6\delta}{\pi^2 (T+1)^2}.
	$
\end{proposition}
\begin{proof}
Let $\|f\|_{K_T^{\vtheta}}$ denote the RKHS norm of $f(\cdot)$ associated with the posterior covariance $\calK_T^{\vtheta}$ of the GP (equation~(\ref{eqn:postCov})). 
From Lemma 7.2 of \cite{Srinivas:2010}, we have
	\bea
	\|\mu_t(\cdot; \vtheta) - f(\cdot)\|_{\calK_T^{\vtheta}} ^2
	&=& \|f\|^2_{\calH_{\vtheta}(\calX)} - \vy^T_T(\vK^{\vtheta}_T + \sigma^2 \vI)^{-1}\vy_T
	+ \sigma^{-2}\|\vepsilon_T\|^2.
	\eea
	This expression, with $\vy_T = \vf_T + \vepsilon_T$, can be easily bounded
	\bea
		\|\mu_T(\cdot; \vtheta) - f(\cdot)\|_{\calK_T^{\vtheta}} ^2
		&=& \|f\|^2_{\calH_{\vtheta}(\calX)} - \vf^T_T(\vK^{\vtheta}_T + \sigma^2 \vI)^{-1}\vf_T
		- 2\vf^T_T(\vK^{\vtheta}_T + \sigma^2 \vI)^{-1}\vepsilon_T \nonumber \\
		&& - \vepsilon^T_T(\vK^{\vtheta}_T + \sigma^2 \vI)^{-1}\vepsilon_T
		+ \sigma^{-2}\|\vepsilon_T\|^2 \nonumber  \\
		&\leq& \|f\|^2_{\calH_{\vtheta}(\calX)} 
		- 2\vf^T_T(\vK^{\vtheta}_T + \sigma^2 \vI)^{-1}\vepsilon_T  
		- \vepsilon^T_T(\vK^{\vtheta}_T + \sigma^2 \vI)^{-1}\vepsilon_T
		+ \sigma^{-2}\|\vepsilon_T\|^2. 
		\nonumber
	\eea
	
	Next, we prove that
	$2\vf^T_T(\vK^{\vtheta}_T + \sigma^2 \vI)^{-1}\vepsilon_T$ 
	and $\sigma^{-2}\|\vepsilon_T\|^2 - 
	\vepsilon^T_T(\vK^{\vtheta}_T + \sigma^2 \vI)^{-1}\vepsilon_T$ are bounded with high probability.

	By Lemma~\ref{lem:c1} we know that 
	$\vf^T_T(\vK^{\vtheta}_T + \sigma^2 \vI)^{-1}\vepsilon_T$ is 
	$\left(\rnorm{f} \right)$-sub-Gaussian. Hence, we can apply the concentration result of 
	Lemma~\ref{lem:subb} to this variable, as follows:
	\bea
	\mathbb{P}\left[ |2\vf^T_T(\vK^{\vtheta}_T + \sigma^2 \vI)^{-1}\vepsilon_T| 
		\geq \sqrt{2\log(\frac{4 (T+1)^2\pi^2}{6\delta})}  \rnorm{f} \right] && \nonumber \\
		&\hspace{-10cm}\leq& \hspace{-5cm}2\exp\left(-\frac{2 \log(\frac{4 (T+1)^2\pi^2}{6\delta}) \rnorm{f}^2}{2 \rnorm{f}^2} \right) \nonumber  \\
		&\hspace{-10cm}=& \hspace{-5cm}\frac{6\delta}{2\pi^2  (T+1)^2 }.
	\eea
	By Lemma~\ref{lem:c2}, with the choice $\eta=\log(\frac{2 (T+1)^2\pi^2}{6\delta}) $, we obtain
	\bea
		\mathbb{P}\left( \frac{\|\vepsilon_T\|^2}{\sigma^{2}} - 
			\vepsilon^T_T(\vK^{\vtheta}_T + \sigma^2 \vI)^{-1}\vepsilon_T
			> 2\gamma^{\vtheta}_T 
			+ \sqrt{8\gamma^{\vtheta}_T \eta } 
			+ 2\sigma \eta \hspace{-1mm} \right)  
		\leq \frac{6\delta}{ 2 \pi^2  (T+1)^2}. \nonumber
	\eea
	Finally, we can use a union bound to combine these two results, yielding 
	\begin{eqnarray*}
		\mathbb{P}\left[ \|\mu_T(\cdot; \vtheta) - f(\cdot)\|_{\calK_T^{\vtheta}}
		\geq \varphi_{T}^{\vtheta} \right] 
		\leq \frac{6\delta}{\pi^2 (T+1)^2}.
	\end{eqnarray*}
\end{proof}

\subsection{Supporting lemmas}

\vspace{4mm}
\begin{lemma}[Lemma 5.3 of ~\cite{Srinivas:2010}]
	\label{lem:53}
	The information gain for the points selected can be expressed in terms of 
	the predictive variances. That is
	$$
		\calI_{\vtheta}(\vy_T; \vf_T) 
		= \frac{1}{2}\sum_{t=1}^{T} \log\left(1 + \sigma^{-2}\sigma_{t-1}(\vx_t; \vtheta) \right)
	$$
\end{lemma}

\vspace{4mm}
\begin{lemma}
	\label{lem:les}
	If $\vtheta' \leq \vtheta$, then $\gamma^{\vtheta}_T \leq \gamma^{\vtheta'}_T$.
\end{lemma}
\begin{proof}
	By definition there exist a set $\calA$ such that 
	$\gamma^{\vtheta}_T = \calI_{\vtheta}(\vy_\calA; \vf_\calA)$. Hence, using Lemma~\ref{lem:53},
	\begin{eqnarray*}
		\gamma^{\vtheta}_T 
		&=& \calI_{\vtheta}(\vy_\calA; \vf_\calA)  \nonumber \\
		&=&  \frac{1}{2}\sum_{\vx_t \in \calA} 
			\log\left(1 + \sigma^{-2}\sigma_{t-1}(\vx_t; \vtheta) \right) \nonumber \\
		&\leq&  \frac{1}{2}\sum_{\vx_t \in A} 
			\log\left(1 + \sigma^{-2}\sigma_{t-1}(\vx_t; \vtheta') \right) \nonumber \\
		&\leq& \max_{\calB\subset \calX: |\calB| = T} \left[\frac{1}{2}\sum_{\vx_t \in \calB} 
			\log\left(1 + \sigma^{-2}\sigma_{t-1}(\vx_t; \vtheta') \right) \right] \nonumber \\
		&=& \gamma^{\vtheta'}_T.
	\end{eqnarray*}
\end{proof}

\vspace{4mm}
\begin{lemma}[Based on Lemma 5.4 of~\cite{Srinivas:2010}]
	\label{lem:54}
	$\sum_{t=1}^{T} \sigma_{t-1}^2(\vx; \vtheta) 
	\leq \frac{2}{\log(1+\sigma^{2})} \gamma_T^{\vtheta}$.
\end{lemma}
\begin{proof}
	Since $s^2 \leq \frac{1}{\sigma^{2}\log(1+\sigma^{2})}\log(1+s^2)$, 
	we have that by Lemma~\ref{lem:53}
	\bea
	\sum_{t=1}^{T} \sigma_{t-1}^2(\vx; \vtheta) 
	&\leq& \sum_{t=1}^{T} \frac{\sigma^{2}}{\sigma^{2}\log(1+\sigma^{2})}
	\log(1+ \sigma^{-2}\sigma_{t-1}^2(\vx; \vtheta)) \nonumber \\
	&=& \frac{2}{\log(1+\sigma^{2})} \calI_{\vtheta}(\vy_T; \vf_T) \nonumber \\
	&\leq& \frac{2}{\log(1+\sigma^{2})} \gamma_T^{\vtheta}.
	\eea
\end{proof}

\vspace{4mm}
\begin{lemma}[Lemma 4 of~\cite{Bull:2011}]
	\label{lem:small}
	If $f \in \cal{H}_{\vtheta}(\calX)$, then $f \in \cal{H}_{\vtheta'}(\calX)$
	for all $0 < \vtheta' \leq \vtheta$ and 
	$$\|f\|_{\cal{H}_{\vtheta'}(\calX)}^2 
	\leq \left(\prod_{i=1}^{d} \frac{\theta_i}{\theta'_i} \right) \rnorm{f}^2.$$
\end{lemma}

\subsection{Properties of the expected improvement acquisition function}

\vspace{4mm}
\begin{lemma}[Based on Lemma 8 of~\cite{Bull:2011}]
	\label{lem:eikey}
	Let $\varphi_{t}^{\vtheta}$ be as defined in Proposition~\ref{prop:bound}
	and $\nu >0$. 
	Assume that $|\mu_{t-1}(\vx, \vtheta) - f(\vx)| \leq \varphi_{t}^{\vtheta}\sigma_{t-1}(\vx; \vtheta)$.
	For $\vx \in \calX$, $t \in \mathbb{N}$,
	set $\mu_{\vtheta}^+ = \max_{\vx \in \calX} 
	\mu_{t-1}(\vx, \vtheta)$, and 
	$I^{\vtheta}_t(\vx) = \max\{0,f(\vx) - \mu_{\vtheta}^+\}$. Then for 
	$$\tau(z) := z \Phi(z)  + \phi(z),$$
	we have that 
	$$
	\max \left(I^{\vtheta}_t(\vx)- \varphi_{t}^{\vtheta}\sigma_{t-1}(\vx; \vtheta), 
	\frac{\tau(-\varphi_{t}^{\vtheta}/\nu)}
	{\tau(\varphi_{t}^{\vtheta}/\nu)}I^{\vtheta}_t(\vx)\right)
	\leq \alpha^\textrm{EI}_{\vtheta}(\vx|\data_{t-1}) 
	\leq I^{\vtheta}_t(\vx) + (\varphi_{t}^{\vtheta} + \nu)\sigma_{t-1}(\vx; \vtheta)
	$$
\end{lemma}
\begin{proof}
	If $\sigma_{t-1}(\vx; \vtheta) = 0$, then 
	$\alpha^\textrm{EI}_{\vtheta}(\vx|\data_{t-1}) = I^{\vtheta}_t(\vx)$
	which makes the result trivial. 
	Thus for the remainder of the proof, we assume that $\sigma_{t-1}(\vx; \vtheta) > 0.$
	Set $q = \frac{f(\vx) - \mu_{\vtheta}^+}{\sigma_{t-1}(\vx; \vtheta)}$,
	and $u = \frac{\mu_{t-1}(\vx, \vtheta) - \mu_{\vtheta}^+}{\sigma_{t-1}(\vx; \vtheta)}$.
	Then we have that 
	$$
	\alpha^\textrm{EI}_{\vtheta}(\vx|\data_{t-1}) 
	= \nu \sigma_{t-1}(\vx; \vtheta) \tau\left(\frac{u}{\nu}\right).
	$$
	By the assumption, we have that 
	$| u - q| < \varphi_{t}^{\vtheta}$.
	As $\tau'(z) = \Phi(z) \in [0, 1]$, $\tau$ is non-decreasing and 
	$\tau(z) \leq 1 + z$ for $z > 0$.
	Hence, 
	\begin{eqnarray}
		\alpha^\textrm{EI}_{\vtheta}(\vx|\data_{t-1}) 
		&\leq& \nu \sigma_{t-1}(\vx; \vtheta) 
			\tau\left(\frac{\max\{0,q\} + \varphi_{t}^{\vtheta}}{\nu}\right) \nonumber \\
		&\leq& \nu \sigma_{t-1}(\vx; \vtheta) 
			\left(\frac{\max\{0,q\} + \varphi_{t}^{\vtheta}}{\nu} + 1\right) \nonumber \\
		&=& I^{\vtheta}_t(\vx) + 
			\left(\varphi_{t}^{\vtheta} + \nu \right)\sigma_{t-1}(\vx; \vtheta) \nonumber
	\end{eqnarray}
	If $I^{\vtheta}_t(\vx) = 0$, then the lower bound is trivial as 
	$\alpha^\textrm{EI}_{\vtheta}(\vx|\data_{t-1})$ is non-negative.
	Thus suppose $I^{\vtheta}_t(\vx) > 0$.
	Since $\alpha^\textrm{EI}_{\vtheta}(\vx|\data_{t-1}) \geq 0$, and $\tau(z) \geq 0$
	for all $z$, and $\tau(z) = z + \tau(-z) \geq z$. 
	Therefore, 
	\begin{eqnarray}
		\alpha^\textrm{EI}_{\vtheta}(\vx|\data_{t-1}) 
		&\geq& \nu \sigma_{t-1}(\vx; \vtheta) 
			\tau\left(\frac{q - \varphi_{t}^{\vtheta}}{\nu}\right) \nonumber \\
		&\geq& \nu \sigma_{t-1}(\vx; \vtheta) 
			\left(\frac{q - \varphi_{t}^{\vtheta}}{\nu} \right) \nonumber \\
		&\geq&  I^{\vtheta}_t(\vx) -
			\varphi_{t}^{\vtheta} \sigma_{t-1}(\vx; \vtheta). 
			\label{eqn:ei1}
	\end{eqnarray}
	Also, as $\tau$ is increasing, 
	\begin{eqnarray}
		\label{eqn:ei2}
		\alpha^\textrm{EI}_{\vtheta}(\vx|\data_{t-1}) \geq 
		\nu \sigma_{t-1}(\vx; \vtheta) \tau\left(\frac{- \varphi_{t}^{\vtheta}}{\nu} \right).
	\end{eqnarray}
	Combining (\ref{eqn:ei1}) and (\ref{eqn:ei2}), we get
	$$
	\alpha^\textrm{EI}_{\vtheta}(\vx|\data_{t-1}) 
	\geq \frac{\nu \tau(-\varphi_{t}^{\vtheta}/\nu)}
	{\varphi_{t}^{\vtheta} + \nu \tau(-\varphi_{t}^{\vtheta}/\nu)}I^{\vtheta}_t(\vx)
	= \frac{\tau(-\varphi_{t}^{\vtheta}/\nu)}
	{\tau(\varphi_{t}^{\vtheta}/\nu)}I^{\vtheta}_t(\vx)
	$$
	which concludes the proof.
\end{proof}

\vspace{4mm}
\begin{lemma}
	\label{lem:10}
	$\left|\mu_{t-1}(\vx_{t}; \vtheta_t) - \mu_{\vtheta_{t}}^+\right| \leq
	\sqrt{\log(t-1+\sigma^2) - \log(\sigma^2)} \nu \sigma_{t-1}(\vx_t; \vtheta_t)$
\end{lemma}
\begin{proof}
	For convenience, define 
	$\vx_{t}^{+} = \argmax_{\vx \in \calX} \mu_{t-1} (\vx; \vtheta_t)$.
	Recall that 
	$\mu^{+}_{\vtheta_t} = \max_{\vx \in \calX} \mu_{t-1} (\vx; \vtheta_t)$.
	Therefore, by the fact that 
	$\alpha^\textrm{EI}_{\vtheta_t}(\vx_{t}|\data_{t-1}) 
	= \max_{\vx \in \calX} \alpha^\textrm{EI}_{\vtheta_t}(\vx|\data_{t-1})$,
	we have
	\begin{eqnarray}
		\label{eqn:ei3}
		&&\nu \sigma_{t-1}(\vx_{t}^+; \vtheta_t) \tau(0) 
		= \alpha^\textrm{EI}_{\vtheta_t}(\vx_{t}^+|\data_{t-1}) \nonumber \\ 
		&&\leq \alpha^\textrm{EI}_{\vtheta_t}(\vx_t|\data_{t-1}) 
		= \nu \sigma_{t-1}(\vx_t; \vtheta_t)
		\tau\left(\frac{\mu_{t-1}(\vx_t, \vtheta_t) - \mu_{\vtheta_t}^+}
		{\nu\sigma_{t-1}(\vx_t; \vtheta_t)}\right),
	\end{eqnarray}
	where $\tau$ is defined as in Lemma~\ref{lem:eikey}.
	We know that $\tau(0) = \frac{1}{\sqrt{2\pi}}$. Thus, equation~(\ref{eqn:ei3}) can be
	re-written as
	\begin{eqnarray} 
		\label{eqn:111}
		\frac{\sigma_{t-1}(\vx_{t}^+; \vtheta_t)}{\sqrt{2\pi}} 
		\leq \sigma_{t-1}(\vx_t; \vtheta_t) 
		\tau\left(\frac{\mu_{t-1}(\vx_t, \vtheta_t) - 
		\mu_{\vtheta_t}^+}{\nu\sigma_{t-1}(\vx_t; \vtheta_t)}\right).
	\end{eqnarray}
	By the definition of $\mu_{\vtheta_t}^+$ 
	we know that $\frac{\mu_{t-1}(\vx_t, \vtheta_t) - 
	\mu_{\vtheta_t}^+}{\nu\sigma_{t-1}(\vx_t; \vtheta_t)} \leq 0$.
	Therefore
	\begin{eqnarray}
		\label{eqn:12}
		\tau\left(\frac{\mu_{t-1}(\vx_t, \vtheta_t) - 
		\mu_{\vtheta_t}^+}{\nu\sigma_{t-1}(\vx_t; \vtheta_t)}\right)
		\leq \phi\left(\frac{\mu_{t-1}(\vx_t, \vtheta_t) - 
		\mu_{\vtheta_t}^+}{\nu\sigma_{t-1}(\vx_t; \vtheta_t)}\right) \nonumber \\
		= \frac{1}{\sqrt{2\pi}} \exp\left( -\frac{1}{2} 
			\left( \frac{\mu_{t-1}(\vx_t, \vtheta_t) - 
			\mu_{\vtheta_t}^+}{\nu\sigma_{t-1}(\vx_t; \vtheta_t)} \right)^2\right).
	\end{eqnarray}
	Combining equations~(\ref{eqn:111}) and~(\ref{eqn:12}), we have
	$$
		\left|\mu_{t}(\vx_t; \vtheta_t) - \mu_{\vtheta_t}^+\right| 
		\leq \sqrt{2\log\left(\frac{\sigma_{t-1}(\vx_t; \vtheta_t)}
		{\sigma_{t-1}(\vx_{t}^+; \vtheta_t)}\right)}
		\nu\sigma_{t-1}(\vx_t; \vtheta_t).
	$$
	Since $\log(\sigma_{t-1}(\vx_t; \vtheta_t)) \leq 0$, 
	it remains to show that 
	$-\log(\sigma_{t-1}^2(\vx_{t}^+; \vtheta_t)) \leq \log(t-1+\sigma^2) - \log{\sigma^2}$.
	To show this, it suffices to show that 
	$\sigma_{t-1}^2(\vx_{t}^+; \vtheta_t) \geq \sigma^2/(t-1+\sigma^2)$.
	To see this, first note that $\sigma_{t-1}^2(\vx_{t}^+; \vtheta_t)$ is minimized
	if $\vx_i = \vx_{t}^+$ $\forall i \leq t-1$.
	That is 
	$$
		\sigma_{t-1}^2(\vx_{t}^+; \vtheta_t) \geq 1 - \mathbf{1}^T(\vJ+\sigma^2 I)^{-1}\mathbf{1}
	$$
	where $\vJ$ is a matrix of all ones.
	Notice that $\vJ$ is of rank $1$.
	Let $\vQ\vSigma \vQ^T$ be the eigen-decomposition of $\vJ$ 
	such that $\Sigma_{1, 1} = \lambda$ where $\lambda$ is the only eigenvalue 
	of $\vJ$. 
	Since $\mathbf{1}$ is an eigenvector of $\vJ$,
	we know that $\lambda = \|\mathbf{1}\|^2_2$ and 
	$\vQ^T\mathbf{1} = [\|\mathbf{1}\|, 0, \cdots, 0]^T$.
	Because $(\vJ+\sigma^2 \vI)^{-1} = \vQ(\vSigma+ \sigma^2 \vI)^{-1} \vQ^T$, we have that
	$\mathbf{1}^T(\vJ+\sigma^2 \vI)^{-1}\mathbf{1} = 
	\frac{\|\mathbf{1}\|^2_2}{\|\mathbf{1}\|^2_2+\sigma^2 }$.
	Therefore 
	$$
		\sigma_{t-1}^2(\vx_{t}^+; \vtheta_t) 
		\geq 1 - \frac{\|\mathbf{1}\|^2_2}{\|\mathbf{1}\|^2_2+\sigma^2 }
		= \frac{\sigma^2}{\|\mathbf{1}\|^2_2+\sigma^2 }
	$$
	which concludes the proof since $\|\mathbf{1}\|^2_2 = t-1$.
\end{proof}

\subsection{Proof of main result}

\begin{proof}[Proof of Theorem~(\ref{thm:main})]

We will need the following definitions 
$\mu^{+}_{\vtheta_t} = \max_{\vx \in \calX} \mu_{t-1} (\vx; \vtheta_t)$ and $\vx_t = \argmax_{\vx \in \calX} 
\alpha^\textrm{EI}_{\vtheta_{t}}(\vx|\data_{t-1})$. 
	By the Cauchy-Schwarz inequality, 
	\bea
		|\mu_{t-1}(\vx, \vtheta_{t}) - f(\vx)| 
		&\leq& \left(\calK^{\vtheta_{t}}_{t-1}(\vx, \vx) \right)^{1/2}\|\mu_{t-1}(\cdot; \vtheta_t) - 
			f(\cdot)\|_{\calK_{t-1}^{\vtheta_t}} \nonumber \\
		&=& \sigma_{t-1}(\vx; \vtheta_{t})\|\mu_{t-1}(\cdot; \vtheta_t) - 
		f(\cdot)\|_{\calK_{t-1}^{\vtheta_t}}. \nonumber 
	\eea
	By Proposition~\ref{prop:bound} and the union bound, 
	we know that $\|\mu_{t-1}(\cdot; \vtheta_t) - 
	f(\cdot)\|_{\calK_{t-1}^{\vtheta_t}} \leq \varphi_{t}^{\vtheta_{t}} $ for all $t \geq 1$
	holds with probability
	at least $1 - \sum_{t=1}^{\infty}\frac{6\delta}{\pi^2 t^2} = 1- \delta$.
	Thus for the remainder of the proof, let us assume that 
	$|\mu_{t-1}(\vx, \vtheta_{t}) - f(\vx)| \leq \varphi_{t}^{\vtheta_{t}}
	\sigma_{t-1}(\vx; \vtheta_{t})$
	$\forall t \in \mathbb{N}, \vx \in \calX$. 
	
	The regret at round $t$ is 
	\begin{eqnarray}
		r_t &=& f(\vx^*) - f(\vx_t) \nonumber \\
		&=& ( f(\vx^*) - \mu_{\vtheta_{t}}^+) - 
			(f(\vx_t) - \mu_{\vtheta_{t}}^+) \nonumber  \\
		&\leq& I^{\vtheta_{t}}_t(\vx^*) +
			\left[ (\mu_{\vtheta_{t}}^+ - \mu_{t-1}(\vx_t; \vtheta_{t})) + 
			\varphi_{t}^{\vtheta_{t}}\sigma_{t-1}(\vx_t; \vtheta_{t}) \right]. 
		\label{eqn:23}
	\end{eqnarray}
	By Lemma~\ref{lem:eikey}, which defines the improvement as 	$I^{\vtheta}_t(\vx) = \max\{0,f(\vx) - \mu_{\vtheta}^+\}$, we know that $I^{\vtheta_{t}}_t(\vx^*) \leq
	\frac{\tau(\varphi_{t-1}^{\vtheta_{t}}/\nu_{t}^{\vtheta_{t}})}
	{\tau(-\varphi_{t-1}^{\vtheta_{t}}/\nu_{t}^{\vtheta_{t}})}
	\alpha^\textrm{EI}_{\vtheta_{t}}(\vx^*|\data_{t-1})$.
	By the assumption on $\left(\nu_{t}^{\vtheta_{t}}\right)^2$,
    there exists a constant $C_3$ such that
    $\left(\frac{\tau(\varphi_{t}^{\vtheta_{t}}/\nu_{t}^{\vtheta_{t}})}
    	{\tau(-\varphi_{t}^{\vtheta_{t}}/\nu_{t}^{\vtheta_{t}})}\right) \leq C_3$.
	By Lemma~\ref{lem:10}, we also have that 
	$(\mu_{\vtheta_{t}}^+ - \mu_{t-1}(\vx_t; \vtheta_t)) 
	\leq \sqrt{\log\left(\frac{t+\sigma^2}{\sigma^2}\right)}  
	\nu_{t}^{\vtheta_{t}} \sigma_{t-1}(\vx_t; \vtheta_{t})$.

	\begin{eqnarray}
		r_t
		&\leq& 
			C_3 \alpha^\textrm{EI}_{\vtheta_{t}}(\vx^*|\data_{t-1})
			+ 
			\left(\sqrt{\log\left(\frac{t+\sigma^2}{\sigma^2}\right)} 
			\nu_{t}^{\vtheta_{t}} + \varphi_{t}^{\vtheta_{t}}\right)
			\sigma_{t-1}(\vx_t; \vtheta_{t})
			\nonumber  \\
		&\leq& 
			C_3 \alpha^\textrm{EI}_{\vtheta_{t}}(\vx_t|\data_{t-1})
			+ \left(\sqrt{\log\left(\frac{t+\sigma^2}{\sigma^2}\right)} 
			\nu_{t}^{\vtheta_{t}} + \varphi_{t}^{\vtheta_{t}}\right)
			\sigma_{t-1}(\vx_t; \vtheta_{t})
			\nonumber  \\
		&\leq&  C_3 \left(I^{\vtheta_{t}}_t(\vx_t) + (\varphi_{t}^{\vtheta_{t}} + 
			\nu_{t}^{\vtheta_{t}})\sigma_{t-1}(\vx; \vtheta_{t})\right) 
			+ \left(\sqrt{\log\left(\frac{t+\sigma^2}{\sigma^2}\right)} 
			\nu_{t}^{\vtheta_{t}} + \varphi_{t}^{\vtheta_{t}}\right)
			\sigma_{t-1}(\vx_t; \vtheta_{t})
			\nonumber  \\
		&\leq& 
			C_3 (\mu_{t-1}(\vx_t, \vtheta_{t}) + 
			\varphi^{\vtheta_{t}}_{t} \sigma_{t-1}(\vx_t; \vtheta_{t}) - \mu_{\vtheta_{t}}^+)^+ 
			\nonumber \\
			&& +  \left( (C_3+1) \varphi_{t}^{\vtheta_{t}} + 
			\left(C_3 + \sqrt{\log\left(\frac{t+\sigma^2}{\sigma^2}\right)} \right) 
			\nu_{t}^{\vtheta_{t}}\right)
			\sigma_{t-1}(\vx_t; \vtheta_{t}) \nonumber  \\
		&\leq& 
			\left( (2C_3+1) \varphi_{t}^{\vtheta_{t}} + 
			\left(C_3+\sqrt{\log\left(\frac{t+\sigma^2}{\sigma^2}\right)} \right) 
			\nu_{t}^{\vtheta_{t}}\right)
			\sigma_{t-1}(\vx_t; \vtheta_{t}).  \label{eqn:bsqr}
	\end{eqnarray}
	Define $(\varphi^L_t)^2 := C_2\|f\|^2_{\calH_{\vtheta^U}(\calX)} 
	+ 2\gamma^{\vtheta^L}_{t-1} +  
	\sqrt{8} \log^{1/2}(2 t^2 \pi^2/3\delta) 
	\left( \sqrt{C_2} \|f\|_{\calH_{\vtheta^U}(\calX)} + \sqrt{\gamma^{\vtheta^L}_{t-1}}\right) + 
	2  \sigma \log(2 t^2\pi^2/3\delta).$
	By Lemma~\ref{lem:small}, we know that 
	$\rnorm{f}^2 \leq C_2 \|f\|^2_{\calH_{\vtheta^U}(\calX)}$.
	Also, by Lemma~\ref{lem:les}, $\gamma^{\vtheta_{t}}_t \leq \gamma^{\vtheta^L}_t$.
	Therefore, 
	$(\varphi_{t}^{\vtheta_{t}})^2 \leq \left(\varphi^L_{t}\right)^2.$
	
	\begin{eqnarray}
		r_t
		&\leq& 
			\left( (2C_3+1) \varphi_{t}^{L} + 
			\left(C_3+\sqrt{\log\left(\frac{t+\sigma^2}{\sigma^2}\right)} \right) 
			\nu_{t}^{\vtheta^L}\right)
			\sigma_{t-1}(\vx_t; \vtheta_{t})  \nonumber  \\
		&\leq& 
			\left( (2C_3+1) \varphi_{t}^{L} + 
			\left(C_3+\sqrt{\log\left(\frac{t+\sigma^2}{\sigma^2}\right)} \right) 
			\nu_{t}^{\vtheta^L}\right)
			\sigma_{t-1}(\vx_t; \vtheta^L).
			\label{eqn:half}  
	\end{eqnarray}

	To simplify notation, let $t_0$ be such that
	$\log\left(\frac{t_0}{\sigma^2}\right) > 1$ and WLOG assume $C_3 > 1$.
	Then

	\begin{eqnarray}
		\sum_{t=t_0}^{T}r_t^2 
		&\leq& \sum_{t=t_0}^{T} (2C_3+1)^2
			\left(\varphi_{t}^L + 
			\sqrt{\log\left(\frac{t}{\sigma^2}\right)}\nu_{t}^{\vtheta^L} \right)^2
			\sigma_{t-1}^2(\vx; \vtheta^L) \nonumber \\
		&\leq& \sum_{t=t_0}^{T} 2(2C_3+1)^2
			\left((\varphi_{t}^L)^2 + 
			\log\left(\frac{t}{\sigma^2}\right) (\nu_{t}^{\vtheta^L})^2\right) 
			\sigma_{t-1}^2(\vx; \vtheta^L) \nonumber \\
		&\leq& 
			2(2C_3+1)^2
			\left((\varphi_{T}^L)^2 + 
			\log\left(\frac{T}{\sigma^2}\right) (\nu_{T}^{\vtheta^L})^2\right) 
			\sum_{t=t_0}^{T} 
			\sigma_{t-1}^2(\vx; \vtheta^L).
			\label{eqn:2222}
	\end{eqnarray}
	By Lemma~\ref{lem:54},  we know that
	$\sum_{t=t_0}^{T} \sigma_{t-1}^2(\vx; \vtheta^L) 
	\leq \frac{2}{\log(1+\sigma^{2})} \gamma^{\vtheta^L}_T$.
	Finally, applying the Cauchy-Schwarz inequality yields 
	$R^2_T \leq T \sum_{t=1}^{T}r_t^2$ thus concluding the proof.
\end{proof}


\end{document}